\documentclass[twoside]{article}

\usepackage[preprint]{aistats2027}
\usepackage[round]{natbib}

\usepackage{microtype}
\usepackage{graphicx}
\usepackage{subcaption}
\usepackage{booktabs} 
\usepackage{hyperref}
\usepackage{makecell}

\usepackage{amsmath}
\usepackage{amssymb}
\usepackage{mathtools}
\usepackage{amsthm}
\allowdisplaybreaks
\usepackage[capitalize,noabbrev]{cleveref}
\usepackage{multirow}

\usepackage{xcolor}

\theoremstyle{plain}
\newtheorem{theorem}{Theorem}[section]
\newtheorem{proposition}[theorem]{Proposition}
\newtheorem{lemma}[theorem]{Lemma}
\newtheorem{corollary}[theorem]{Corollary}
\theoremstyle{definition}

\newtheorem{assumption}[theorem]{Assumption}
\theoremstyle{remark}

\usepackage{comment}

\renewcommand{\d}{\mathrm{d}}

\newcommand{\PP}{\mathbb{P}}
\newcommand{\RD}{\mathbb{R}^{d}}

\newcommand{\EE}{\mathbb{E}}

\newcommand{\iid}{\overset{\mathrm{i.i.d.}}{\sim}}

\usepackage[textwidth=1.8cm, textsize=tiny]{todonotes}

\begin{document}

%

%

\twocolumn[

\aistatstitle{Learning Transition Kernels of Jump-Diffusion Processes with Conditional Diffusion Models}

\aistatsauthor{ Yuzhen Zhao  \And Yating Liu\And Quentin Guibert }

\aistatsaddress{ CEREMADE, CNRS, Universite Paris-Dauphine, PSL, 75016 Paris, France \\ \texttt{yuzhen.zhao@dauphine.eu; liu@ceremade.dauphine.fr; guibert@ceremade.dauphine.fr}} ]

\begin{abstract}
We study the problem of learning transition kernels for time-homogeneous jump-diffusion processes using conditional diffusion models, with the goal of generating new sample paths from training data consisting of \(N\) independent trajectories observed on a high-frequency discrete time grid. On the theoretical side, we establish non-asymptotic bounds for the conditional score estimation error and for the KL divergence between the laws of the true and generated discretely observed paths. On the numerical side, we first evaluate our method on synthetic data to assess the theoretical findings and benchmark its performance against the approach of \cite{gao2025generating}. We then apply our method to real-world data and investigate its performance on a probabilistic forecasting task.
\end{abstract}


\section{Introduction}


\textbf{Motivation.} Jump-diffusion processes provide a natural framework for modeling stochastic systems characterized by both continuous fluctuations and sudden discontinuous changes. They are widely used across a range of application domains, including mathematical finance and insurance \citep{MERTON1976125,kou2002jump,contFinancialModellingJump2003, asmussenRuinProbabilities2010}, energy markets \citep{cartea2005pricing}, and neuronal dynamics \citep{melanson2019data}. Statistical inference for jump-diffusion processes is therefore a fundamental topic, and has been extensively studied, in particular with classical approaches focusing on nonparametric estimation of the drift, diffusion, and jump characteristics \citep{SCHMISSER2019, amorino2022nonparametric}. Related work has also considered learning the infinitesimal generator from data, see \cite{kostic2024learning} for diffusion processes without jumps.  Recent advances in conditional generative modeling offer an alternative approach by directly learning the transition kernel of the process \citep{liu2025training,gao2025generating}.


\textbf{Model.} 
  On a filtered probability space
\((\Omega,\mathcal F,(\mathcal F_t)_{t\geq 0},\mathbb P)\)
satisfying the usual conditions, we consider an \(\mathbb{R}^d\)-valued 
stochastic process \(X=(X_t)_{t\in[0,T]}\) satisfying the following stochastic differential equation (SDE)
\begin{equation}\label{eq:sde}
    \d X_t  \!  =  \! b(X_{t-})\,\d t  \! + \!  \sigma(X_{t-})\,\d W_t \!  + \!\!  \int_E \!\gamma(X_{t-},z)\,\Pi(\d t,\d z),
\end{equation}
where $X_0$ has initial distribution $\mu_0$, \(W=(W_t)_{t\in[0,T]}\) is a standard  \(d\)-dimensional \((\mathcal F_t)\)-Brownian motion independent of $X_0$, $E\subseteq \RD$ is a Borel set,  $b:\mathbb{R}^d\to\mathbb{R}^d,\, 
\sigma:\mathbb{R}^d\to\mathbb{R}^{d\times d}
$, $
\gamma:\mathbb{R}^d\times E\to\mathbb{R}^d$ are measurable functions,  and \(\Pi(dt,dz)\) is a Poisson random measure on \([0,\infty)\times E\) with $\lambda := \nu(E) < \infty$, independent of \((W,X_0)\), and with compensator \(\d t\,\nu(\d z)\). 



\textbf{Objective.} This paper studies conditional diffusion models for learning the discrete-time transition kernel of \eqref{eq:sde} over a time step \(\Delta > 0\), defined for every \(x\in\mathbb{R}^d\), \(A\in\mathcal{B}(\mathbb{R}^d)\) by
\begin{equation}\label{eq:transition-kernel-def}
\bar{P}_{\Delta}(x,A)
\coloneqq
\mathbb{P}\left(X_{t+{\Delta}}\in A \,|\, X_t=x\right).
\end{equation}
Equivalently, we can reformulate the problem in terms of learning the increment kernel, defined for every \(x\in\mathbb{R}^d\), \(B\in\mathcal{B}(\mathbb{R}^d)\) by
\begin{equation}\label{eq:transition-kernel-bis}
P_{\Delta}(x,B)
\coloneqq
\mathbb{P}\big( (X_{t+{\Delta}}-X_t)\in B \,|\, X_t=x\big),
\end{equation}
since \(P_{\Delta}(x,B)=\bar{P}_{\Delta}(x,x+B)\), where $x+B\coloneqq\{y\in\RD\,\mid\, y-x\in B\}$.  By time homogeneity, neither \(\bar{P}_{\Delta}\) nor \({P}_{\Delta}\) depends on \(t\). Whenever it exists, let \(p_\Delta(\cdot\mid x)\) denote the 
density of \(P_\Delta(x,\cdot)\) with respect to Lebesgue measure. Since this paper focuses on learning the transition kernel, we assume that the initial distribution \(\mu_0\) of \(X_0\) is known.



\textbf{Training data.} For the training data, we assume that \(N \in \mathbb{N}^\star\) i.i.d. sample paths \(X^{(n)}\), \(1\le n\le N\), are available on $[0,T]$, each with the same law as \(X\) and observed at high frequency on the uniform grid
\begin{equation}\label{eq:train-set}
   \bar{X}^{(n)}_{t_0 : t_M} := \big(  X^{(n)}_{t_0=0},  X^{(n)}_{t_1}, \dots,  X^{(n)}_{t_M=T} \big), \;1\!\leq\! n\!\leq \!N, 
\end{equation}
where $t_m\!=\!m \Delta$ with time step $\Delta \!=\! \frac{T}{M}$. For each $n=1,\ldots,N,\, m=0,\ldots,M-1$, we define the observed increment
\begin{equation}
\label{eq:observed-increment}
{Y}_m^{(n)}\coloneqq{X}_{t_{m+1}}^{(n)}-{X}_{t_m}^{(n)}.
\end{equation}
The training set can then be written as
\begin{equation}\label{eq:flattened-data}
\mathcal D_N=
\Big\{
\big({X}_{t_m}^{(n)},{Y}_m^{(n)}\big)_{0\leq m\leq M-1}, \, 
1\leq n\leq N
\Big\}.\nonumber
\end{equation}


\textbf{Related work and positioning.}
Recently, generative models have been increasingly applied to time-series forecasting and generation, see, e.g., \cite{shen2023non,kollovieh2025flow,demarco2026}. Despite their strong empirical performance reported in the above references, establishing theoretical guarantees and convergence rates remains challenging in general time-series settings. More recently, \cite{gao2025generating} initiated a theoretical study of generative modeling for Markovian SDEs. Our work follows this line of research but differs from \cite{gao2025generating} in two main respects. First, while \cite{gao2025generating} focuses on continuous diffusion processes, we consider the class of time-homogeneous jump-diffusion processes with the aim of being able to generate processes with jumps. Second, rather than learning a separate generative model at each time step, we exploit time homogeneity and the Markov property to learn a single conditional generative model for the transition kernel $P_{\Delta}(x, B)$ defined in \eqref{eq:transition-kernel-bis}. This construction allows observations from different time steps to be pooled for training a shared model, thereby  enlarging the effective training dataset and improving numerical performance.

\textbf{Contribution of this paper.}
This paper studies the problem of learning transition kernels for time-homogeneous jump-diffusion processes using conditional diffusion models.

On the theoretical side, we establish non-asymptotic bounds for the conditional score estimation error and for the KL divergence between the laws of the true and generated discretely observed paths.
 
On the numerical side, we first evaluate our method on synthetic data to assess the empirical convergence rate and compare its performance with the approach of \cite{gao2025generating}. We then apply our method to real-world data and investigate its performance on a probabilistic forecasting task.

\textbf{Organization of this paper.}
Section \ref{sec:diff-model} describes the construction of the conditional diffusion model considered in this paper. Section \ref{sec:theoretical-part} establishes theoretical non-asymptotic error bounds with respect to the KL divergence between the laws of the true and generated discretely observed paths. Section \ref{sec:numerical} presents two numerical experiments illustrating the performance of the proposed method. Finally, Section \ref{sec:conclusion} concludes the paper. Additional experimental details and complete proofs are provided in the Supplementary Material.

\textbf{Notation.} For every $n\in\{1, \ldots , N\}$, let $\mathcal Z^{(n)}=( X_{t_m}^{(n)},  Y_{m}^{(n)})_{0\leq m\leq M-1}$ denote the $n$-th observed sample path in the dataset $\mathcal D_N$.  Let $\mu_m$ be the law of ${X}_{t_m}$ for every $m\in\{0,\ldots, M\}$ and let $\bar\mu_M:=\frac1M\sum_{m=0}^{M-1}\mu_m .$ For every bounded measurable function
\(\ell:\mathbb{R}^d\times\mathbb{R}^d\to\mathbb{R}\), we write $\mathbb{E}_{U\sim\bar{\mu}_M} \mathbb{E}_{V\sim P_\Delta(U,\cdot)}  \left[\ell(U,V)\right]  = \frac{1}{M}\sum_{m=0}^{M-1}\mathbb{E}\left[\ell\left(X_{t_m}, X_{t_{m+1}}-X_{t_m}\right)\right].$ 
For a vector $v$ and a matrix $A$, $\|v\|=\|v\|_2$ and  $\|A\|=\|A\|_F$ denote the Euclidean and Frobenius norms respectively, while \(  \|v\|_\infty := \max_i |v_i| \) and \(\|A\|_{\max} := \max_{i,j}|A_{ij}|.\) Moreover, $\|\cdot\|_0$ denotes the number of nonzero entries. 


\section{Conditional Diffusion Models for Increment Generation}\label{sec:diff-model}

We use a conditional diffusion model (see e.g. \cite{song2020score,tang2025score})  to learn and sample from the conditional law $P_{\Delta}(x,\cdot)$ defined in \eqref{eq:transition-kernel-bis}. We present the main steps in this section  for our setting and further details can be found in \cite{song2020score,tang2025score}.



\textbf{Forward noising process.} For a terminal noising time $\mathcal{T}>0$, and for each \(x\in\mathbb{R}^d\), we define the forward noising process 
\(U^x=(U_\tau^x)_{\tau\in[0,\mathcal{T}]}\) as the solution of the variance-preserving (VP) SDE
\begin{equation}\label{eq:forward-noising-process}
    \d U_\tau^x   =-f(\tau)U_\tau^x\,\d \tau+g(\tau)\,\d B_\tau,
\end{equation}
where $U_0^x$ has distribution $P_\Delta(x,\cdot)$, $f(\tau)=\tfrac{1}{2}\beta(\tau)$, $g(\tau)=\sqrt{\beta(\tau)}$, $ \beta(\tau)=a_0\tau+b_0$ with $a_0,b_0>0$, and \(B=(B_\tau)_{\tau\in[0,\mathcal{T}]}\) is a
\(d\)-dimensional Brownian motion independent of \(U_0^x\). Equation \eqref{eq:forward-noising-process} admits an explicit solution and for every fixed \(\tau>0\), 
\begin{equation}
\label{eq:forward-conditional-law}
    U_\tau^x \mid U_0^x=u_0
    \sim
    \mathcal{N}
    \left(
        \alpha_\tau u_0,
        \rho_\tau^2 I_d
    \right),
\end{equation}
with  $\alpha_\tau
   = \exp\left(
    -\frac{a_0}{4}\tau^2
    -\frac{b_0}{2}\tau
    \right),
    \rho_\tau^2
    = 1-\alpha_\tau^2$ (see e.g. \citet[Eq. (16)]{tang2025score}).
    
For \(\tau>0\), let \(q_{\tau\mid0}(u\mid u_0)\) denote the conditional density of \(U_\tau^x\) given \(U_0^x=u_0\). By \eqref{eq:forward-conditional-law},
\begin{equation}
\label{eq:perturbation-kernel}
    q_{\tau\mid0}(u\mid u_0)
    =
    \frac{1}{(2\pi\rho_\tau^2)^{d/2}}
    \exp\left(
        -\frac{
            \lVert u-\alpha_\tau u_0\rVert^2
        }{
            2\rho_\tau^2
        }
    \right).
\end{equation}
Let \(q_\tau(\cdot\mid x)\) denote the density of \(U_\tau^x\).
By conditioning on \(U_0^x\), we obtain
\begin{equation}
\label{eq:noised-conditional-density}
    q_\tau(u\mid x)
    =
    \int_{\mathbb{R}^d}
    q_{\tau\mid0}(u\mid u_0)\,
    P_\Delta(x,\d u_0).
\end{equation}
Whenever \(P_\Delta(x,\cdot)\) admits a density
\(p_\Delta(\cdot\mid x)\), we set  $q_0(u\mid x)=p_\Delta(u\mid x).$ Moreover, for every fixed $\tau>0$, $ U_{\tau}^x$ has the same law as $ \alpha_{\tau}U_0^x + \sqrt{1-\alpha_{\tau}^2}\,\xi$ by \eqref{eq:forward-conditional-law}, where \(\xi \sim\mathcal{N}(0,I_d)\) and is independent of \(U_0^x\). 

Since in \eqref{eq:forward-conditional-law}, \(\alpha_\tau\to0\) and \(\rho_\tau\to1\) as
\(\tau\to\infty\), the terminal distribution $q_{\mathcal T}(\cdot\mid x)$ approaches the standard Gaussian distribution as $\mathcal T$ increases. We therefore use
\(\mathcal{N}(0,I_d)\) as the reference distribution for initializing the reverse-time generative process.

\textbf{Conditional score.}  A key component of the diffusion-model approach is to estimate the conditional score of \(q_\tau(\cdot\mid x)\), defined for \(\tau>0\) by
\begin{equation}
\label{eq:diffusion-score}
    s^\star(\tau,u,x) := \nabla_u\log q_\tau(u\mid x).
\end{equation}

By \eqref{eq:perturbation-kernel}, we have $    \nabla_u\log q_{\tau\mid0}(u\mid u_0) = -\frac{u-\alpha_\tau u_0}{\rho_\tau^2}.$
Conditioning on \(U_\tau^x=u\) yields (see e.g. \citet[Theorem 4.3]{tang2025score})
\begin{equation}
\label{eq:conditional-score-identity}
    s^\star(\tau,u,x)   =  \mathbb{E}  \left[      -\frac{u-\alpha_\tau U_0^x}{\rho_\tau^2}     \,\middle|\,     U_\tau^x=u \right].
\end{equation}

Let \(K\geq1\) be the number of noising-time steps and consider the uniform grid
\begin{equation}
\label{eq:diffusion-time-grid}
    h:=\tfrac{\mathcal{T}}{K},
    \qquad
    \tau_k:=kh,
    \qquad
    k=0,\ldots,K.
\end{equation}
Since \(\rho_{\tau_0}=0\), score estimation is performed at the
positive noising times \(\tau_1,\ldots,\tau_K\). For each \(k=1,\ldots,K\), define $   \xi_k
    :=
    \frac{
        U_{\tau_k}^x-\alpha_{\tau_k}U_0^x
    }{
        \rho_{\tau_k}
    }.$
Then \(\xi_k\sim\mathcal{N}(0,I_d)\) and \(\xi_k\) is independent of
\(U_0^x\). Consequently, 
\begin{equation}\label{eq:score-star-equiv}
    s^\star(\tau_k,u,x) = \mathbb{E}
    \left[
        -\frac{\xi_k}{\rho_{\tau_k}}
        \,\middle|\,
        U_{\tau_k}^x=u
    \right].
    \end{equation}

\textbf{Denoising score-matching loss.}
For any measurable function $    s:
    \{\tau_1,\ldots,\tau_K\}
    \times\mathbb{R}^d
    \times\mathbb{R}^d
    \to\mathbb{R}^d,$
following \citet[Eq. (37)]{tang2025score}, we define the pairwise denoising score-matching (DSM) loss function by 
\begin{align}
\label{eq:pairwise-dsm-loss}
   & \ell^{\mathrm{DSM}}(s;x,y):=\\
  &  \frac{1}{K} \sum_{k=1}^{K} \mathbb{E}_{\xi_k\sim\mathcal{N}(0,I_d)}\!
    \left[
        \left\|
            \rho_{\tau_k}
            s\left(
                \tau_k,
                \alpha_{\tau_k}y+\rho_{\tau_k}\xi_k,
                x
            \right)
            +
            \xi_k
        \right\|_2^2
    \right].\nonumber
\end{align}
Moreover, we define the DSM risk by
\begin{equation}
    \mathcal{L}^{\mathrm{DSM}}(s)
    :=
    \mathbb{E}_{X\sim\bar{\mu}_M}
    \mathbb{E}_{Y\sim P_\Delta(X,\cdot)}
    \left[
        \ell^{\mathrm{DSM}}(s;X,Y)
    \right],\nonumber
\end{equation}
and based on the training data \(\mathcal{D}_N\), the empirical
DSM loss is then defined by
\begin{equation}
\label{eq:empirical-dsm-loss}
    \mathcal{L}^{\mathrm{DSM}}_N(s)
    :=
    \frac{1}{NM}
    \sum_{n=1}^{N}
    \sum_{m=0}^{M-1}
    \ell^{\mathrm{DSM}}
    \left(
        s;
        {X}_{t_m}^{(n)},
        {Y}_m^{(n)}
    \right).
\end{equation}

\textbf{Estimation function class.}
Let $\mathcal{F}(L,\mathbf{p})$ denote the class of functions defined by neural networks with $L$ hidden layers, and layers width given by \(\mathbf p=(p_0,\ldots,p_{L+1})\in\mathbb N^{L+2}\) with
\(p_0=2d\) and \(p_{L+1}=d\). Each $s_\theta \in \mathcal{F}(L,\mathbf{p})$ is a function from   $\{\tau_1,\ldots,\tau_K\}\times\mathbb{R}^d\times\mathbb{R}^d$  
to $\mathbb{R}^d$ defined by 
\begin{align}\label{eq:score-network}
  &   h_0^{(k)}(u,x)  :=  (u,x), \nonumber
    \\
& \text{and for }  0\leq \ell\leq L-1,\nonumber\\
   & h_{\ell+1}^{(k)}(u,x):= \text{ReLU}  \left(       \Gamma_{\ell,k}    \left(        W_\ell h_\ell^{(k)}(u,x)+v_\ell   \right) \right),\nonumber
    \\
   & s_\theta(\tau_k,u,x):=  \Gamma_{L,k}  \left(      W_Lh_L^{(k)}(u,x)+v_L\right),
\end{align}
where 
$
W_\ell\in
\mathbb{R}^{p_{\ell+1}\times p_\ell},\,
v_\ell\in\mathbb{R}^{p_{\ell+1}},\,$ and $
\Gamma_{\ell,k} = \operatorname{diag}(\gamma_{\ell,k}),\,\gamma_{\ell,k}\in\mathbb{R}^{p_{\ell+1}},\,
\ell=0,\ldots,L,\,
k=1,\ldots,K
$. The weights \(W_\ell\) and biases \(v_\ell\) are shared across all diffusion-time indices \(k\), while the diagonal scaling matrices \(\Gamma_{\ell,k}\) depend on \(k\).

Let $F\geq 1$. We define the neural network class
\begin{align}
\label{eq:score-network-class}
&\mathcal{F}(L,\mathbf{p},S,G,F)
:=\Bigg\{s_\theta \in \mathcal{F}(L,\mathbf{p})
\text{ such that } \nonumber\\
&\max_{1\leq k\leq K}
\sup_{(u,x)\in\mathbb{R}^{2d}}
\left\|
    s_\theta(\tau_k,u,x)
\right\|_\infty
\leq F,\nonumber
\\
& 
\max_{0\leq\ell\leq L}
\left(
    \lVert W_\ell\rVert_{\max}
    \vee
    \lVert v_\ell\rVert_\infty
\right)
\leq1,
\\
& \max_{\substack{
    0\leq\ell\leq L\\
    1\leq k\leq K
}}
\lVert\Gamma_{\ell,k}\rVert_{\max}
\leq G,
\quad
\sum_{\ell=0}^{L}
\left(
    \lVert W_\ell\rVert_0
    +
    \lVert v_\ell\rVert_0
\right)
\leq S
\Bigg\}.\nonumber
\end{align}
Since each \(\Gamma_{\ell,k}\) is diagonal, the total number of
trainable parameters is bounded by $   S_{\mathrm{tot}}  :=S+K\sum_{\ell=0}^{L}p_{\ell+1}.$

For any measurable score function  \(s : \{\tau_1,\ldots, \tau_K\} 
\times\mathbb R^d\times\mathbb R^d
\rightarrow
\mathbb R^d\), define its score estimation
error by
\begin{align}
\label{eq:score-estimation-error}
 &   \mathcal{E}(s) :=  \frac{1}{K}  \sum_{k=1}^{K}
 \mathbb{E}_{{X\sim\bar{\mu}_M},{U_{\tau_k}^X\sim q_{\tau_k}(\cdot\mid X)}} \Big[ \nonumber\\
 &\hspace{1.8cm}      \left\|       s(\tau_k,U_{\tau_k}^X,X)    -  s^\star(\tau_k,U_{\tau_k}^X,X)  \right\|_2^2 \Big].
\end{align}

Let
$s_{\widehat{\theta}_N}
\in
\mathcal{F}(L,\mathbf{p},S,G,F)$
denote the score estimator obtained by minimizing the empirical DSM loss \(\mathcal{L}^{\mathrm{DSM}}_N\).  
Recall that \(s_{\widehat{\theta}_N}\) depends on the training sample \(\mathcal{D}_N\). A key step in the theoretical analysis is therefore to establish an upper bound on the expected score estimation error $\mathbb{E}_{\mathcal{D}_N}\big[\mathcal{E}\big(s_{\widehat{\theta}_N}\big)\big],$ 
as detailed in Section  \ref{subsec:global-score-error}.

\textbf{Reverse process and recursive path generation}  Given \(\mathcal D_N\) and \(x\in\mathbb R^d\), let $\widehat Z_0^x\sim\mathcal N(0,I_d),
$
and for \(k=0,\ldots,K-1\), set
\begin{align}
\label{eq:reverse-euler-sampler}
\widehat Z_{k+1}^x
=&
\widehat Z_k^x+g(\tau_{K-k})\sqrt h\,\zeta_k+
\\
&
\left[
f(\tau_{K-k})\widehat Z_k^x
+
g^2(\tau_{K-k})
s_{\hat\theta_N}(\tau_{K-k},\widehat Z_k^x,x)
\right]h, \nonumber
\end{align}
where $f,g$ are defined in \eqref{eq:forward-noising-process}, \(\zeta_0,\ldots,\zeta_{K-1} \iid \mathcal N(0,I_d)\), independent of \(\mathcal D_N\) and
\(\widehat Z_0^x\)  \cite[Eq. (72)]{tang2025score}. For every $x\in\RD$ and $B\in\mathcal{B}(\RD)$, define
\begin{equation}
\label{eq:generated-increment-law}
\widehat Q_{N,\Delta}(B\mid x)
:=
\PP\big(\widehat Z_K^x\in B\mid x,\mathcal D_N\big).
\end{equation} 
Then $\widehat Q_{N,\Delta}(B\mid x)$ provides an estimator of the increment kernel $P_{\Delta}(x, B)$ defined by \eqref{eq:transition-kernel-bis}. 

Given the estimated increment kernel  $\widehat Q_{N,\Delta}(B\mid x)$, a new discrete sample path can be generated recursively, conditional on the training data $\mathcal{D}_N$. First, sample $ \widehat{X}_{t_0}\sim\mu_0.$ Then, for \(m=0,\ldots,M-1\), sample
\begin{equation}\label{eq:construction-new-path}
     \widehat{Y}_m\mid\widehat{X}_{t_m}\!\sim\widehat{Q}_{N,\Delta}\big(   \cdot\mid\!\widehat{X}_{t_m}\big),
    \text{ and set } \widehat{X}_{t_{m+1}}\! =\!\widehat{X}_{t_m}+\widehat{Y}_m.   
\end{equation}
Thus, we obtain a generated discrete sample path $(\widehat X_{t_0},\ldots,\widehat X_{t_M}).$



\section{Theoretical Analysis}\label{sec:theoretical-part}

Let $P_{0:M}$
denote the law of the true discretely observed path $( X_{t_0},\ldots,X_{t_M} )$, and let $\widehat{P}_{0:M}^{\,N}$
denote the law of the generated path $\big(\widehat{X}_{t_0},\ldots,\widehat{X}_{t_M}\big)$ conditional on the training data $\mathcal{D}_N$.  For two probability measures \(\mu,\nu\), recall that the Kullback--Leibler (KL) divergence is defined by
\[
\mathrm{KL}(\mu\|\nu):=\begin{cases}
\int \log\big(\tfrac{d\mu}{d\nu}\big)\,\d \mu,& \mu\ll\nu,\\
+\infty,& \text{otherwise}.
\end{cases}
\]
Since \(\widehat P_{0:M}^{\,N}\) depends on \(\mathcal D_N\), we measure the path-level approximation error through the expected KL divergence
\begin{equation}
\label{eq:main-risk}
\mathcal R_N
:=
\mathbb E_{\mathcal D_N}
\left[
\mathrm{KL}
\left(
P_{0:M}
\middle\|
\widehat P_{0:M}^{\,N}
\right)
\right].
\end{equation}
The goal of this section is to then derive a non-asymptotic upper bound of $\mathcal R_N$. 

This section is organized as follows. We first state the assumptions in Section \ref{subsec:assumption} and introduce localization and clipping of the score in Section \ref{subsec:local-clip}. We then establish the clipped and global score estimation error bound in Section \ref{subsec:global-score-error}, and use it to control the KL divergence between the true and generated path laws in Section \ref{subsec:KL-gene-path}. All detailed proofs are provided in the Supplementary Material.

\subsection{Assumptions}\label{subsec:assumption}

\begin{assumption}
\label{ass:sde-wellposed} 
$\EE[\|X_0\|^2]<\infty$ and $\nu(E)<\infty$. 
The coefficient functions \(b\) and \(\sigma\) are globally Lipschitz continuous and there exists a measurable function $L_\gamma : E\rightarrow [0,+\infty)$ such that, for every $x,y\in\RD, \,z\in E$, 
\[\Vert \gamma (x,z)-\gamma(y,z)\Vert \leq L_\gamma(z) \Vert x-y\Vert.\]
Moreover, $\int_E
\bigl[L_\gamma(z)^2+\|\gamma(0,z)\|^2\bigr]
\,\nu(\d z)<\infty.$
\end{assumption}

Assumption~\ref{ass:sde-wellposed} ensures that
SDE~\eqref{eq:sde} admits a unique nonexplosive adapted
c\`adl\`ag strong solution (see e.g. \citet[Theorem 6.2.3 and Theorem 6.2.9]{MR2512800}). The following lemma is a direct consequence of
\citet[Corollary~6.2.4]{MR2512800} and establishes properties needed to control the localization tails in the following
Section~\ref{subsec:global-score-error}.
\begin{lemma}
\label{lem:coefficient-increment-integrability}
 Suppose Assumption~\ref{ass:sde-wellposed} holds. 
Then, for the fixed grid \(t_m=m\Delta\), the averaged state and increment second moments are finite, i.e.,
\begin{align}
\label{eq:state-second-moment}
&M_{X,2}
:=\int_{\mathbb R^d}
\|x\|^2\,\bar\mu_M(\d x)
<\infty, \\
&M_2
:=
\int_{\mathbb R^d}
\int_{\mathbb R^d}
\|y\|^2 P_\Delta(x,dy)\,\bar\mu_M(\d x)
<\infty. \label{eq:increment-second-moment}
\end{align}

\end{lemma}

\begin{assumption}
\label{ass:increment-density-regularity}
For the fixed observation step size \(\Delta>0\), the increment kernel
\(P_\Delta\) admits a  density function $p_\Delta$ w.r.t the Lebesgue measure, that is, 
\[
P_\Delta(x,A)
=
\int_A p_\Delta(y\mid x)\,\d y,
\qquad
A\in\mathcal B(\mathbb R^d).
\]
For \(\bar\mu_M\)-almost every \(x\in\mathbb R^d\), the following hold:
\begin{enumerate}
\item
The map \(y\mapsto p_\Delta(y\mid x)\) is twice continuously differentiable
on \(\mathbb R^d\), and $p_\Delta(y\mid x)>0$ for $y\in\mathbb R^d.$

\item
The conditional increment score
\[
s_\Delta^\star(y,x)
:=
\nabla_y\log p_\Delta(y\mid x)
\]
is Lipschitz continuous in \(y\). That is, there exists a measurable function
\(L_\Delta:\mathbb R^d\to[0,\infty)\) such that
\[
\left\|
s_\Delta^\star(y,x)-s_\Delta^\star(y',x)
\right\|
\le
L_\Delta(x)\|y-y'\|,
y,y'\in\mathbb R^d .
\]
Moreover, $L_1
:=\int_{\mathbb R^d} L_\Delta(x)\,\bar\mu_M(\d x)
<\infty .$
\end{enumerate}
\end{assumption}

Assumption~\ref{ass:increment-density-regularity} is imposed
to apply \citet[Proposition~2]{gao2025generating}.
A simple example satisfying this assumption is the
one-dimensional additive jump-diffusion $\d X_t=b\,\d t+\sigma \,\d W_t+c\,\d N_t$, where $b,c\in\mathbb R$ and $\sigma>0$ are constants, and $N$ is a Poisson
process with intensity $\lambda>0$, independent of the Brownian motion $W$. 





\subsection{Localization and Clipping}\label{subsec:local-clip}
To control the global score estimation error, we first study the error
of the clipped score estimator on a bounded region.  For \(R>0\), we define $\mathcal X_R
:=
\left\{
x\in\mathbb R^d:\|x\|\le R
\right\}$ and $\mathcal U_R
:=
\left\{
u\in\mathbb R^d:\|u\|\le R
\right\}.$ The localization region is then defined by
\[
\mathsf D_R
:=
[\tau_1,\mathcal{T}]\times\mathcal U_R\times\mathcal X_R.
\]
Moreover, set $\rho_{\min}
:=
\min_{1\le k\le K}\rho_{\tau_k}>0.$
For \(k=1,\ldots,K\), define the clipped score by
\begin{align}
\label{eq:clipped-population-score}
&s_F^\star(\tau_k,u,x)
:=
\\
&\mathbb E
\left[
\left(
1\wedge
\frac{F\rho_{\tau_k}}{\|\xi\|}
\right)
\left(
-\frac{\xi}{\rho_{\tau_k}}
\right)
\middle|
\alpha_{\tau_k}Y+\rho_{\tau_k}\xi=u,\ X=x
\right],\nonumber
\end{align}
where $X$ is sampled w.r.t. $\bar\mu_M$, $Y\,|\,X=x\sim P_\Delta(x,\cdot),$ $\xi\sim\mathcal N(0,I_d),$ and \(\xi\) is independent of \((X,Y)\). By construction,
\[
\left\|
s_F^\star(\tau_k,u,x)
\right\|
\le F.
\]
For any measurable score function
\[
s:
\{\tau_1,\ldots,\tau_K\}
\times\mathbb R^d\times\mathbb R^d
\longrightarrow
\mathbb R^d,
\]
define the localized clipped score estimation error by
\begin{align}
\label{eq:localized-clipped-score-error}
&\mathcal E_{R,F}(s)
:=
\frac1K
\sum_{k=1}^K
\mathbb E_{X\sim\bar\mu_M}
\mathbb E_{U_{\tau_k}^X\mid X\sim q_{\tau_k}(\cdot\mid X)}
\bigg[\\
&\left\|
s(\tau_k,U_{\tau_k}^X,X)
-
s_F^\star(\tau_k,U_{\tau_k}^X,X)
\right\|^2
\mathbf 1_{\{
(\tau_k,U_{\tau_k}^X,X)\in\mathsf D_R
\}}
\bigg].\nonumber
\end{align}
We now define the clipping term $\mathcal C_{R,F}$ and tail term $\mathcal T_{R,F}$ as follows.\begin{align}
\label{eq:clipping-bias}
&\mathcal C_{R,F}
:=
\frac1K
\sum_{k=1}^K
\mathbb E_{X\sim\bar\mu_M}
\mathbb E_{U_{\tau_k}^X\mid X\sim q_{\tau_k}(\cdot\mid X)}
\bigg[\\
&
\left\|
s_F^\star(\tau_k,U_{\tau_k}^X,X)
-
s^\star(\tau_k,U_{\tau_k}^X,X)
\right\|^2 
\mathbf 1_{\{
(\tau_k,U_{\tau_k}^X,X)\in\mathsf D_R
\}}
\bigg],\nonumber\\
\label{eq:tail-term}
&\mathcal T_{R,F}
:=
\frac1K
\sum_{k=1}^K
\mathbb E_{X\sim\bar\mu_M}
\mathbb E_{U_{\tau_k}^X\mid X\sim q_{\tau_k}(\cdot\mid X)}
\bigg[\\
&\quad 
\left(
2dF^2
+
2
\left\|
s^\star(\tau_k,U_{\tau_k}^X,X)
\right\|^2
\right) 
\mathbf 1_{\{
(\tau_k,U_{\tau_k}^X,X)\notin\mathsf D_R
\}}
\bigg].\nonumber
\end{align}
Localization and clipping allow us to decompose the score estimation error $\mathcal E(s)$ into several terms that can be controlled separately. More precisely, we have the following lemma.
\begin{lemma}
\label{lem:tail-clipping-reduction}
For every $s\in\mathcal F(L,\mathbf p,S,G,F)$,
\begin{equation}
\label{eq:tail-clipping-reduction}
\mathcal E(s)
\le
2\mathcal E_{R,F}(s)+2\,\mathcal C_{R,F}+\mathcal T_{R,F}.
\end{equation}
\end{lemma}
In the next subsection, we derive an upper bound for each term on the right-hand side of \eqref{eq:tail-clipping-reduction}.


\subsection{Clipped and Global Score Error}\label{subsec:global-score-error}

Recall that
\(s_{\widehat{\theta}_N}\in\mathcal{F}(L,\mathbf{p},S,G,F)\)
denotes the score estimator obtained by minimizing the empirical DSM loss
\(\mathcal{L}^{\mathrm{DSM}}_N\) defined by \eqref{eq:empirical-dsm-loss}. 
The goal of this subsection is to derive an upper bound for its global score estimation error. We proceed in two steps. First, we control the localized clipped score error
\(\mathbb{E}_{\mathcal D_N}\!\big[\mathcal E_{R,F}(s_{\widehat{\theta}_N})\big]\).
To this end, we introduce a localized excess-risk functional $\mathcal{G}_s^w$ and derive an oracle-type bound that decomposes the error into approximation, optimization, and statistical estimation terms (see Proposition \ref{prop:localized-score-oracle}). We then control the error introduced by clipping and localization through the clipping and tail terms respectively (see  Lemma \ref{lem:clipping-tail-control}), which allows us to pass from the localized clipped error to the global score estimation error (see  Theorem \ref{thm:score-error-bound}).

Let $\mathsf{T}_{F,k}(\xi)\coloneqq \big(1\wedge\frac{F\rho_{\tau_k}}{\|\xi\|}\big)\big(-\frac{\xi}{\rho_{\tau_k}}\big)$ denote the clipped regression target appearing in \eqref{eq:clipped-population-score}. For a fixed observed path $z\!=\!(x_m, y_m)_{m=0}^{M-1},$ we define the noisy input by $U_{m,k}(\xi)=\alpha_{\tau_k}y_m+\rho_{\tau_k}\xi$, and define the weighted localized excess loss of $s$ relative to $s_F^{\star}$ by
\begin{equation}
\label{eq:block-excess}
\begin{aligned}
\mathcal{G}_s^w(z):=&\frac1{MK}\sum_{m=0}^{M-1}\sum_{k=1}^K\rho_{\tau_k}^2\mathbb E_{\xi\sim\mathcal N(0,I_d)}\Big[\nonumber\\
&\Big(\left\|s(\tau_k,U_{m,k}(\xi),x_m)-\mathsf{T}_{F,k}(\xi)\right\|^2\\
&\quad-\left\|s_F^\star(\tau_k,U_{m,k}(\xi),x_m)-\mathsf{T}_{F,k}(\xi)\right\|^2
\Big)\nonumber\\
&\times\mathbf 1_{\{(\tau_k,U_{m,k},x_m)\in\mathsf D_R\}}\Big].
\end{aligned}
\end{equation}
Since $s_{\hat\theta_N}$ obtained by numerical optimization may not be an exact minimizer of the localized empirical excess loss, we introduce an optimization error term to quantify the gap between the empirical criterion achieved by $s_{\hat\theta_N}$ and the best value attainable over the network class $\mathcal F(L,\mathbf p,S,G,F)$. More precisely, we define
\begin{align}\label{eq:expected-localized-optimization-error}
\Psi_{N,R,F}^{w}(s_{\hat\theta_N})
:=&
\mathbb E_{\mathcal D_N}
\Big[
\frac1N
\sum_{n=1}^N
\mathcal{G}_{s_{\hat\theta_N}}^w(\mathcal Z^{(n)})
\\
&-
\inf_{s\in\mathcal F(L,\mathbf p,S,G,F)}
\frac1N
\sum_{n=1}^N
\mathcal{G}_s^w(\mathcal Z^{(n)})
\Big].\nonumber
\end{align}
The next proposition provides an upper bound for the localized clipped score error $\mathbb E_{\mathcal D_N}
\left[\mathcal E_{R,F}(s_{\hat\theta_N})\right]$. 
\begin{proposition}
\label{prop:localized-score-oracle}
Assume Assumptions \ref{ass:sde-wellposed} holds. Let $ \mathsf{p}_{\max}\coloneqq \max_{0\le\ell\le L+1}p_\ell$ and let $\displaystyle \Lambda_{N,R,F}\coloneqq \log\big(
C N d\big[
F^2+F(1+R)(1\vee G)(L+1)
(1+\mathsf{p}_{\max})
\big]
\big)$
There exists a universal constant $C>0$ such that
\begin{align}
&\mathbb E_{\mathcal D_N}
\left[\mathcal E_{R,F}(s_{\hat\theta_N})\right]
\le
\frac{C}{\rho_{\min}^2}\Big[
\inf_{s\in\mathcal F(L,\mathbf p,S,G,F)}
\mathcal E_{R,F}(s)
\nonumber\\
&\hspace{1cm}+\Psi_{N,R,F}^w(s_{\hat\theta_N})+dF^2\frac{S_{\mathrm{tot}}(L+1)}{N}\Lambda_{N,R,F}
\Big].\nonumber
\end{align}
\end{proposition}
The following lemma controls the clipping bias and localization tail error.
\begin{lemma}
\label{lem:clipping-tail-control}
Under Assumption~\ref{ass:sde-wellposed}, there exists a universal
constant $C>0$ such that 
\begin{align}
  & \mathcal C_{R,F}
\!\le\!
C\rho_{\min}^{-2}d
\exp(
-\tfrac14\max\{F\rho_{\min}-\sqrt d,0\}^2
), \nonumber\\
&\mathcal T_{R,F}
\le C\big[ dF^2(M_{X,2}+M_2+d)R^{-2}\nonumber
\\\label{eq:tail-upper-bound}
&\hspace{2cm} +\rho_{\min}^{-2}d
(M_{X,2}+M_2+d)^{1/2}R^{-1}
\big].
\end{align}



\end{lemma}

We are now ready to establish the global score error bound. For convenience, let $\mathfrak S_{N,R,F}$ denote the quantity that summarizes  the preceding upper bounds, namely,
\[
\begin{aligned}
&\mathfrak S_{N,R,F}
:=
\tfrac{1}{\rho_{\min}^2}
\Big[
\inf_{s\in\mathcal F(L,\mathbf p,S,G,F)}
\mathcal E_{R,F}(s)
\\
&
+\Psi_{N,R,F}^{w}(s_{\hat\theta_N})+dF^2\frac{S_{\mathrm{tot}}(L+1)}{N}
\Lambda_{N,R,F}
\Big]
\\
&+\rho_{\min}^{-2}d
\exp\left(
-\frac14\max\{F\rho_{\min}-\sqrt d,0\}^2
\right)
\\
&+\rho_{\min}^{-2}d
(M_{X,2}+M_2+d)^{1/2}R^{-1}
\\
&+dF^2(M_{X,2}+M_2+d)R^{-2}.
\end{aligned}
\]
\begin{theorem}
\label{thm:score-error-bound}
Assume Assumption~\ref{ass:sde-wellposed} holds. There exists a universal
constant \(C>0\) such that
\begin{equation}
\label{eq:score-error-bound}
\mathbb E_{\mathcal D_N}
\left[
\mathcal E(s_{\hat\theta_N})
\right]
\le
C\mathfrak S_{N,R,F}.
\end{equation}
\end{theorem}

The following corollary specifies the choice of the clipping and localization parameters \(F\) and \(R\) so that the clipping and tail terms vanish as the number \(N\) of sample paths in the training set \(\mathcal D_N\) increases.

\begin{corollary}
\label{cor:score-error-rate}
Under the assumptions of Theorem~\ref{thm:score-error-bound},
suppose in addition that there exists a constant
$C_{\mathrm{width}}>0$, independent of $N$ and $d$, such that, 
\[
\max_{0\le\ell\le L+1}p_\ell
\le C_{\mathrm{width}}d.
\]
For $N\ge2$, set
$F_N:=\frac{\sqrt d+2\sqrt{\log N}}{\rho_{\min}}\vee 1, 
R_N:=N,$
and let
$s_{\hat\theta_N}\in\mathcal F(L,\mathbf p,S,G,F_N)$
be the score estimator.
Then there exists a constant $C>0$, depending only on
$\rho_{\min}$, $L$, $G$, and $C_{\mathrm{width}}$, such that 
\begin{align}
&\mathbb E_{\mathcal D_N}
\left[\mathcal E(s_{\hat\theta_N})\right]
\le C\Big[
\inf_{s\in\mathcal F(L,\mathbf p,S,G,F_N)}
\mathcal E_{R_N,F_N}(s)\nonumber\\
&\quad +\Psi_{N,R_N,F_N}^{w}(s_{\hat\theta_N})
+\frac{d(d+\log N)}{N}\Big\{
\nonumber\\
&\quad S_{\mathrm{tot}}(L+1)\log(N+d) +\sqrt{M_{X,2}+M_2+d}
\Big\}\Big].\nonumber
\end{align}
\end{corollary}

\subsection{KL Bounds for Generated Path Laws}\label{subsec:KL-gene-path}
We now translate the score estimation error established in Theorem \ref{thm:score-error-bound} into a one-step KL divergence bound for the transition kernel and derive the corresponding path-level KL bound. 

\begin{theorem}
\label{thm:path-kl-bound}
Under Assumptions~\ref{ass:sde-wellposed}
and~\ref{ass:increment-density-regularity}, there exists a constant
$C>0$, depending only on the fixed parameters $a_0$, $b_0$,
and $\mathcal T$, such that, conditional on the training sample $\mathcal{D}_N$, the following holds 
\begin{align}
\label{eq:one-step-kl-bound}
&\mathbb E_{X\sim\bar\mu_M}\left[
\mathrm{KL}\left(
P_\Delta(X,\cdot)
\middle\|
\widehat Q_{N,\Delta}(\cdot\mid X)
\right)\right]\\
&\le
C\left\{
\eta_{\mathcal T,h}\left[M_2+d(1+L_1)\right]
+(\mathcal T+1)^3\mathcal E(s_{\hat\theta_N})
\right\},\nonumber
\end{align}
where $\eta_{\mathcal{T},h}:=\exp\left(-\frac{a_0}{2}\mathcal{T}^2-b_0\mathcal{T} \right) + h(\mathcal{T}+1)^5$. 
Moreover, the KL risk between the true and generated path laws satisfies
\begin{equation}
\mathcal R_N
\le
CM\left\{
\eta_{\mathcal T,h}\left[M_2+d(1+L_1)\right]
+(\mathcal T+1)^3\mathfrak S_{N,R,F}
\right\}.\nonumber
\end{equation}
\end{theorem}

\section{Numerical Experiments}\label{sec:numerical}

In this section, we present two numerical experiments to evaluate  our method. Section~\ref{subsec:numerical-cvg} considers synthetic data and investigates the empirical convergence rate of the method.  Section \ref{subsec:real-data} presents an application of the proposed method to a real-world dataset to evaluate its forecasting performance. Detailed hyperparameter settings for both experiments are provided in the Supplementary Material.

\subsection{Empirical convergence rate }\label{subsec:numerical-cvg}

We consider the following \(d\)-dimensional SDE on \([0,1]\)
\begin{equation}
\d X_t=-X_{t-}\,\d t+\d W_t+\int_{(0,\infty)^d}\gamma(X_{t-},z)\,\Pi(\d t,\d z)\nonumber
\end{equation}
where $X_0\sim\mathcal{N}(0,I_d)$, $\gamma(x,z)=0.1x+\big( 1+\tfrac{1}{\sqrt{1+\lVert x\rVert^2}}\big)z,$  and
the Poisson random measure \(\Pi(\d t,\d z)\), independent of \((W_t)_{t\in[0,1]}\), has intensity
\(10\,e^{-(z_1+\cdots+z_d)}\,\d z\,\d t\) on \([0,1]\times(0,\infty)^d\).
The process is observed on the grid \(t_m=m\Delta\), \(m=0,\ldots,100\), with \(\Delta=0.01\).

\textbf{Main hyperparameter setting.} We use a score network with \(K=100\) noising steps and noise scale \(\beta(\tau)=19.9\tau+0.1\). The score network is trained for 256 epochs using the Adam optimizer with a learning rate of \(10^{-3}\) and a  batch size of \(256\). We select the network with the lowest denoising score-matching loss on the validation set. The network architecture $\mathbf p$ and sparsity ratio  \(s_{\mathrm{ratio}}\)  are selected from 
\begin{align}
    &\mathbf p\in\big\{(2d,16,32,16,d),\ (2d,16,32,32,16,d),\nonumber\\
   & \quad \quad \quad(2d,16,32,64,32,16,d)\big\},\nonumber\\
   &s_{\mathrm{ratio}}\in\{0.25,0.50,0.75\}.\nonumber
\end{align}
The sparsity level \(S\) is then obtained by multiplying \(s_{\mathrm{ratio}}\) by the total number of parameters in the linear layers. The complete hyperparameter settings are provided in the Supplementary Material.


\textbf{Experimental setting}
We consider dimensions \(d\in\{1,2\}\) and different training sample sizes $N\in\{100,200,500,1000,2000,5000\}.$
For each combination of $d$ and $N$, we run 20 independent trials.
In each trial, the model is trained on $N$ sample paths, and the validation and test sets each consist of 3000 independently simulated paths, both independent of the training set.

\textbf{Illustration of the generated paths.} 
Figure \ref{fig:our-method-vs-gao} compares the paths generated by our method with those generated by \cite{gao2025generating}.  Qualitatively, 
the paths generated by our method exhibit more pronounced  jumps and appear more similar to those in the test set than those generated by \cite{gao2025generating}.
\begin{figure}[ht]
    \centering
    \begin{subfigure}[t]{0.48\linewidth}
        \centering
        \includegraphics[width=\linewidth]{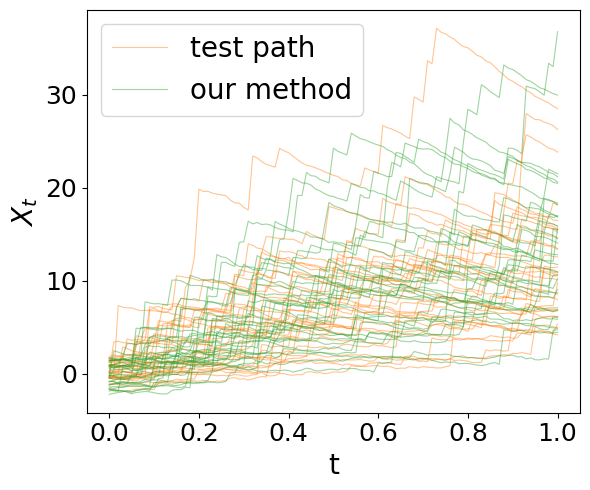}
    \end{subfigure}
    \hfill
    \begin{subfigure}[t]{0.48\linewidth}
        \centering
        \includegraphics[width=\linewidth]{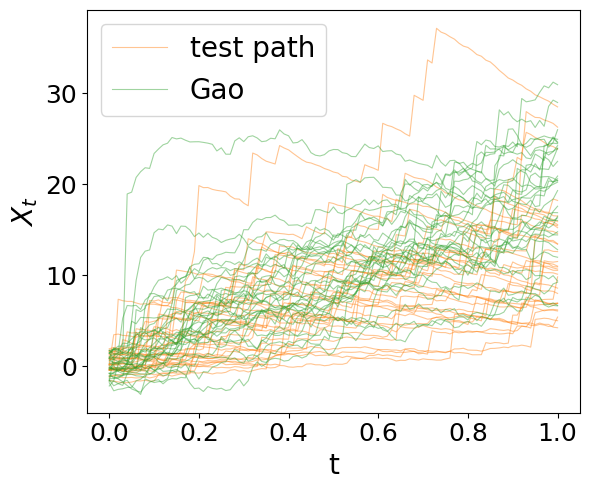}
    \end{subfigure}
\caption{Comparison of test sample paths with sample paths generated by our method (left) and by the method in \cite{gao2025generating} (right).}
\label{fig:our-method-vs-gao}
\end{figure}

\textbf{Empirical convergence rate.}
We next assess the empirical convergence rate of our method on the test set  and compare it with that of \cite{gao2025generating}. Recall that Section \ref{sec:theoretical-part} uses the KL divergence to measure the discrepancy between the true and generated path distributions. Directly estimating the KL divergence over the full path, viewed as a vector in $\mathbb R^{101d}$, is challenging due to its high dimensionality. We therefore use the following procedure based on the KL divergence estimator of \cite{MR2729888}.


For each replication, we draw \(100\) subsets of size \(20/d\) from the time indices \(\{1,\ldots,100\}\)  without replacement. For each subset, the states at the selected time points are  concatenated into a \(20\)-dimensional vector. The corresponding  KL divergence is then estimated using  \(3000\) test paths, the \(3000\) generated paths and the method proposed by \cite{MR2729888}. For each replication, we average the KL estimates over the \(100\) randomly selected subsets. The reported value is then the mean over the \(20\) independent replications, with a \(95\%\) Student-\(t\) confidence interval.

Figure \ref{fig:exp1-cvg-rate} plots the estimated error described above  against the training sample size \(N\) on a log-log scale, showing an empirical rate close to $N^{-1/2}$ up to logarithmic factors. In Theorem~\ref{thm:path-kl-bound}, the upper bound involves \(\mathfrak S_{N,R,F}\), whose convergence rate is studied in Corollary~\ref{cor:score-error-rate}. The corresponding bound contains approximation and optimization errors,  and a statistical estimation term of order $N^{-1}$ up to logarithmic factors. The empirical rate $N^{-1/2}$ observed here may therefore be due to the approximation or optimization errors dominating in the finite-sample regime. In addition, the numerical estimation of the KL divergence using the method of \cite{MR2729888} may introduce an additional estimation error.  



\begin{figure}[ht]
    \centering
    \begin{subfigure}[t]{0.48\linewidth}
        \centering
        \includegraphics[width=\linewidth]{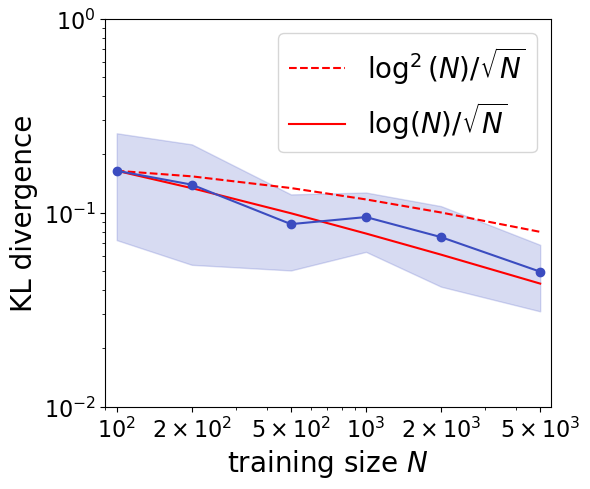}
    \end{subfigure}
    \hfill
    \begin{subfigure}[t]{0.48\linewidth}
        \centering
        \includegraphics[width=\linewidth]{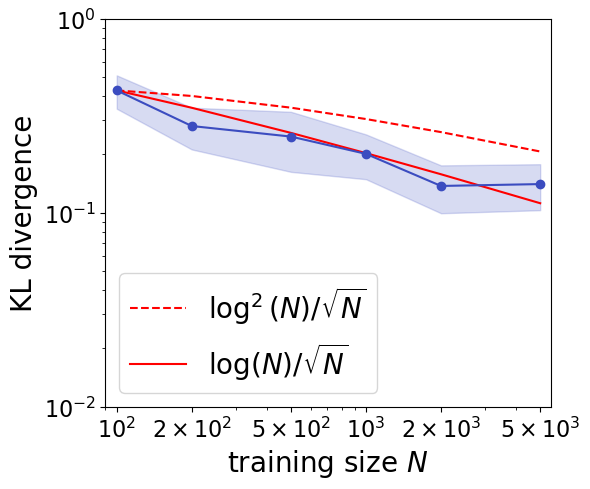}
    \end{subfigure}
\caption{Empirical convergence rates for \(d=1\) (left) and \(d=2\) (right).}
    \label{fig:exp1-cvg-rate}
\end{figure}

\begin{figure}[ht]
    \centering
    \begin{subfigure}[t]{0.48\linewidth}
        \centering
        \includegraphics[width=\linewidth]{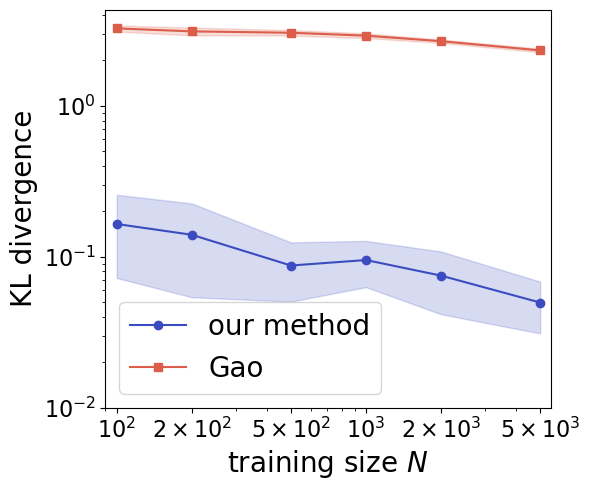}
    \end{subfigure}
    \hfill
    \begin{subfigure}[t]{0.48\linewidth}
        \centering
        \includegraphics[width=\linewidth]{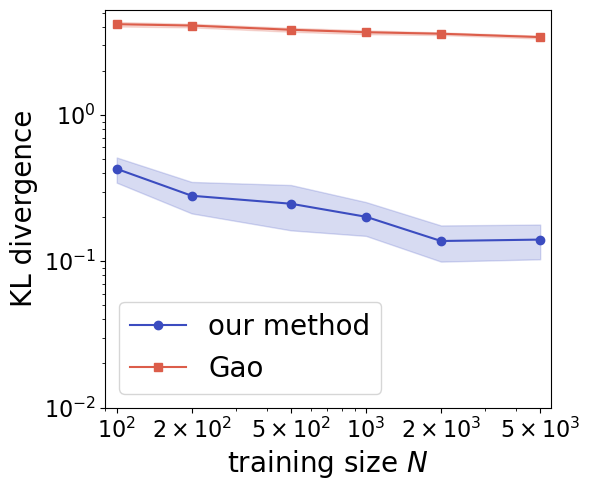}
    \end{subfigure}
\caption{Comparison of the empirical convergence rates for sample paths generated by our method and by \cite{gao2025generating} for \(d=1\) (left) and \(d=2\) (right).}
\label{fig:our-method-vs-gao-2}
\end{figure}

Figure \ref{fig:our-method-vs-gao-2} compares the empirical convergence rates of our method and that of \cite{gao2025generating} on a log-log scale, showing faster convergence of our method in both settings. Although the confidence intervals for the estimator of
\cite{gao2025generating} appear narrower in
Figure~\ref{fig:our-method-vs-gao-2}, this is due to the logarithmic
scale used in the figure and its larger mean KL estimates.
Its empirical variance across replications is in fact larger
than that of our method in every setting. Detailed values underlying the figure are provided in Table~\ref{tab:kl-our-vs-gao} in the Supplementary Material.





\subsection{Application to Forecasting on Real Data}\label{subsec:real-data}

In this section, we apply our method to real-world datasets from finance and insurance, and evaluate its forecasting performance.  We only report results for two datasets (\textbf{Crypto} and \textbf{Danish}) here. Additional results on other datasets are provided in the Supplementary Material.

\textbf{Crypto}: daily closing prices of BTC, ETH, SOL and XRP quoted in USDT, obtained from the Binance spot API, from 23~March 2021 to 26~October 2025. The four assets are modeled jointly ($d=4$).

\textbf{Danish}: a surplus path obtained from the \texttt{danishuni} fire-loss records in the R package \texttt{fitdistrplus} \citep{delignette2015fitdistrplus}. Losses are summed by day, with days without recorded losses assigned a value of zero.  The surplus process starts at zero with a premium rate set equal to the average daily loss.

\textbf{Training--validation--test split.} Each series is split chronologically into training (\(70\%\)), validation (\(20\%\)), and test (\(10\%\)) sets, used respectively to fit the model, select the best hyperparameters, and evaluate forecasting performance. 


\textbf{One-step-ahead forecasting procedure.} We evaluate one-step-ahead forecasts on the test set. Conditioning on each observed state $X_t$, we sample an increment \(\widehat{\Delta X}_t\) from the estimated increment kernel $\widehat Q_{N,\Delta}(\cdot\mid X_t)$ defined in \eqref{eq:generated-increment-law}, and predict the next state by $X_t+\widehat{\Delta X}_t$. The ``mean'' forecast shown in Figure~\ref{fig:real-flat} is $X_t+\frac{1}{20}\sum_{l=1}^{20}\widehat{\Delta X}_t^{(l)}$  based on 20 independent increment draws at each observed $X_t$. The blue band in Figure~\ref{fig:real-flat} represents the 95\% Student-$t$ confidence interval for the mean forecast $X_t+\overline{\Delta X}_t\pm t_{0.975,19}\,\frac{s_t}{\sqrt{20}}$,
where \(s_t\) is the sample standard deviation of the \(20\) draws. For the \textbf{Crypto} dataset, we report the results for BTC and ETH here, and those for SOL and XRP in the Supplementary Material.

\begin{figure}[ht]
    \centering
    \begin{subfigure}{\linewidth}
        \centering
        \includegraphics[width=0.82\linewidth]{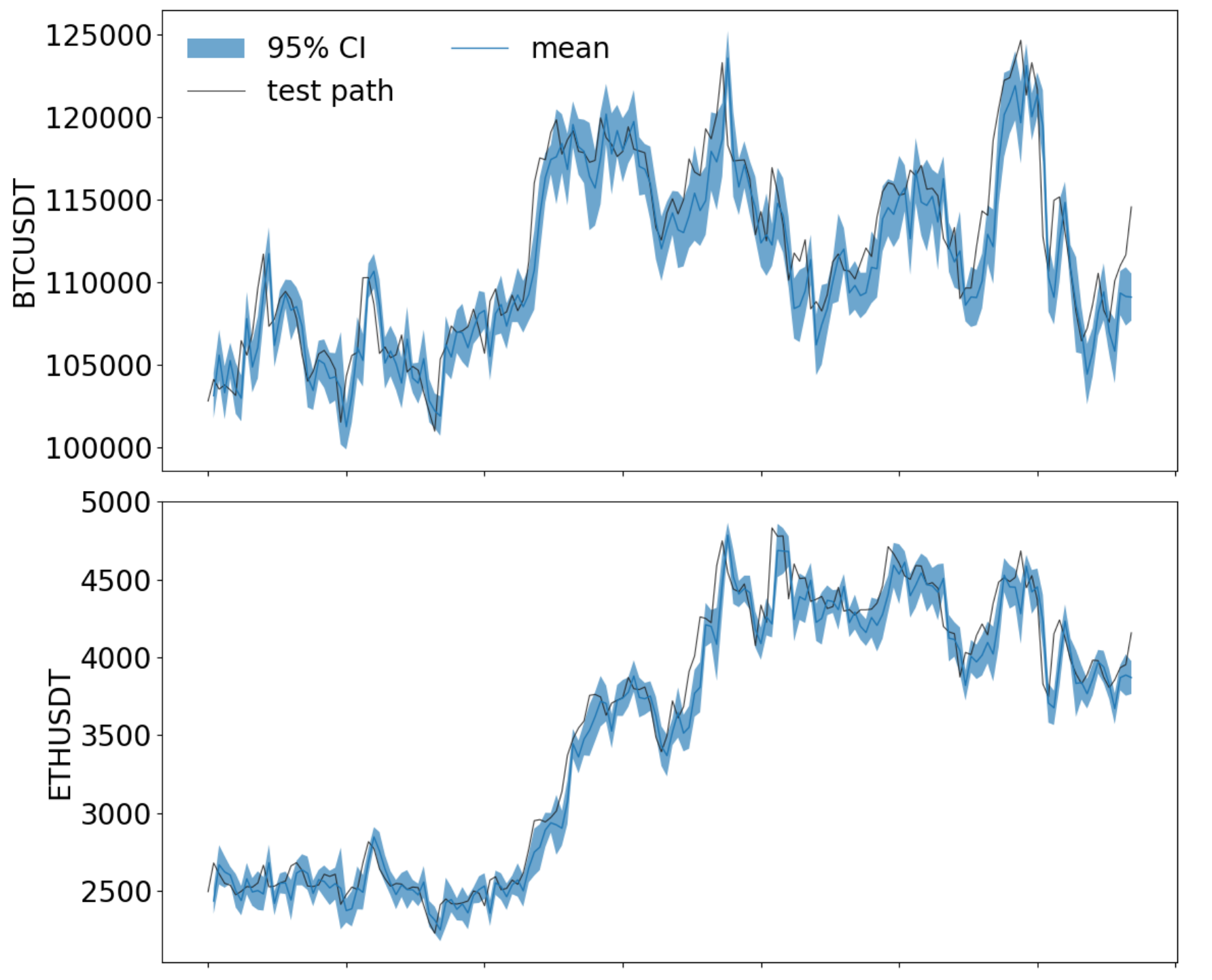}
    \end{subfigure}
    \begin{subfigure}{\linewidth}
        \centering
        \includegraphics[width=0.82\linewidth]{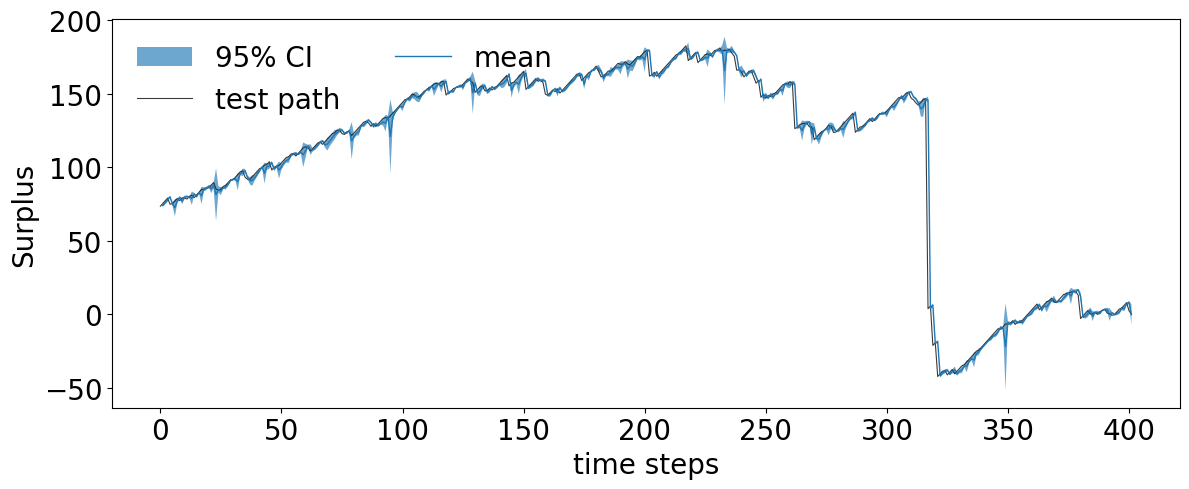}
    \end{subfigure}
    \caption{One-step forecasts of our estimator on the test segment for BTC and ETH in the \textbf{Crypto} dataset (top) and for the \textbf{Danish} dataset (bottom).}
    \label{fig:real-flat}
\end{figure}
\textbf{Benchmark method.} To compare our method with another generative approach, we use the Schrödinger bridge (SB) method of \cite{demarco2026} as a forecasting benchmark.  
Further implementation details are provided in the Supplementary Material.


\begin{figure}[ht]
    \centering
    \begin{subfigure}{\linewidth}
        \centering
        \includegraphics[width=0.7\linewidth]{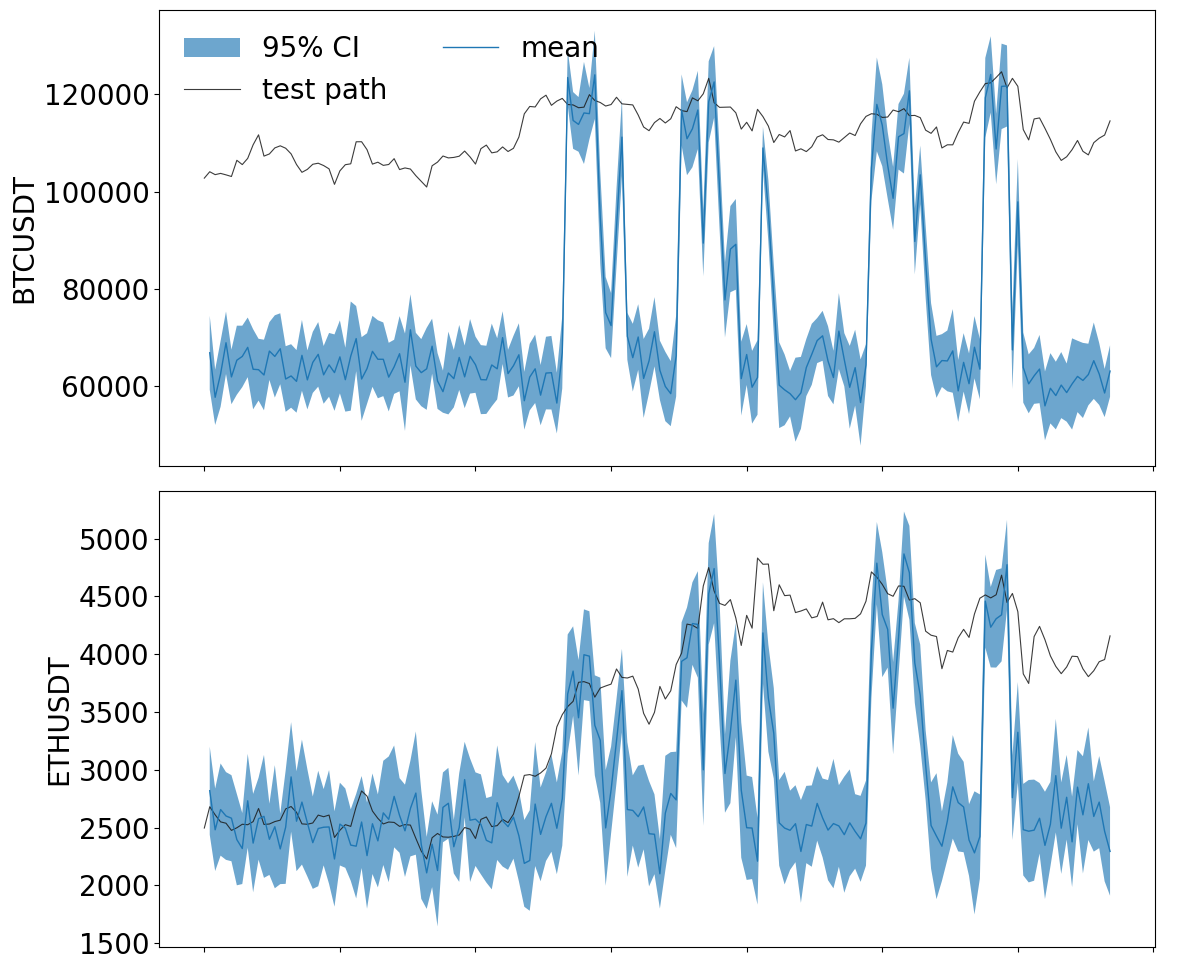}
    \end{subfigure}
    \begin{subfigure}{\linewidth}
        \centering
        \includegraphics[width=0.7\linewidth]{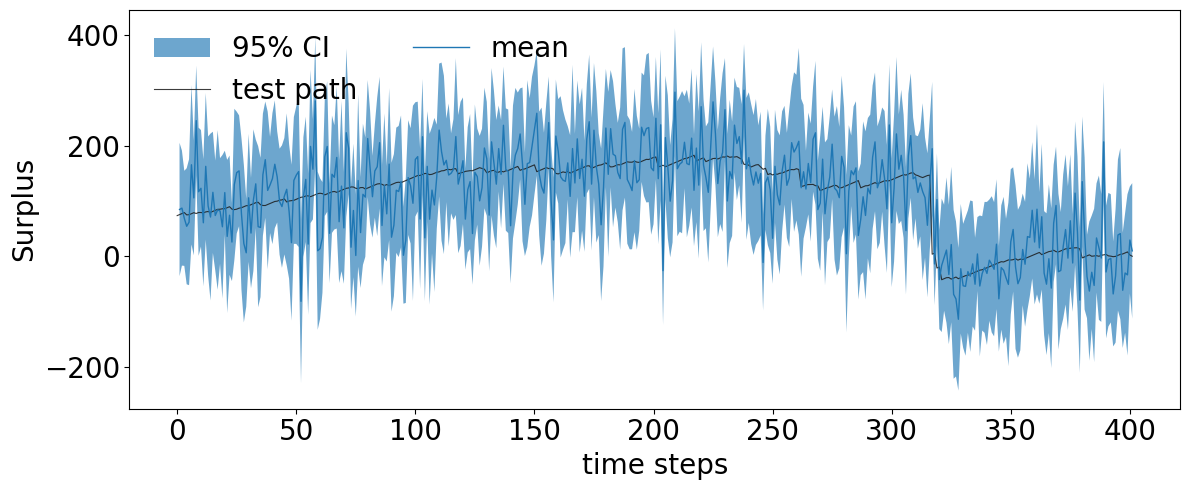}
    \end{subfigure}
    \caption{One-step forecasts of the time-homogeneous SBJTS estimator  on the test segment for BTC and ETH in the \textbf{Crypto} dataset (top) and for the \textbf{Danish} dataset (bottom).  The curves and the band are constructed as in Figure~\ref{fig:real-flat}.}
    \label{fig:real-sbjts}
\end{figure}

On the \textbf{Crypto} dataset, our forecasts stays within the
range of the observed test trajectory, whereas the SBJTS method forecasts lie below it. On the \textbf{Danish} dataset, our forecasts remain close to the observed path, while the SBJTS forecasts show no clear level shift but exhibit much larger fluctuations. The conditional
kernel regression estimators used in SBJTS method in \citep{demarco2026} rely on local training-data coverage around the conditioning states. Otherwise, the ratios defining
these estimators are undefined and require additional numerical
handling. 

\section{Conclusion}\label{sec:conclusion}

We study conditional diffusion models for learning the discrete-time transition kernel of the time-homogeneous jump-diffusion SDE \eqref{eq:sde}. We establish non-asymptotic bounds for the score error and the KL divergence between true and generated discrete path laws. Numerical experiments show improved performance over \cite{gao2025generating} and illustrate the effectiveness of our method for one-step forecasting on real-world data. 

\textbf{Limitation.} Our analysis assumes observations on a fixed time grid with step size $\Delta>0$ and establishing bounds uniformly over $\Delta$ remains challenging. We also assume that the initial distribution $\mu_0$ is known. When $\mu_0$ is unknown, using a smoothed estimator $\widehat\mu_0$ would instead add the term $\mathrm{KL}(\mu_0|\widehat\mu_0)$ to the path-level bound in Theorem \ref{thm:path-kl-bound}. Despite these limitations, the numerical results show promising performance for both transition-kernel learning and one-step-ahead forecasting task, and extending our approach to multi-step forecasting is a promising direction for future research.

\subsection*{Acknowledgment}
Yuzhen Zhao and Quentin Guibert acknowledge support from the Research Chair DIALog under the aegis of the Risk Foundation, a joint initiative by Universit\'e Paris-Dauphine PSL and CNP Assurances. The funding body had no role in the design of the study, the analysis, the writing of the manuscript, or the decision to submit the article for publication. Yating Liu acknowledges financial support from the CNRS through the MITI interdisciplinary programs. Yating Liu and Yuzhen Zhao acknowledge Marc Hoffmann for fruitful discussions on this work.


\bibliographystyle{plainnat}
\bibliography{main-bib}

\clearpage
\appendix
\thispagestyle{empty}

\onecolumn
\aistatstitle{Supplementary Materials for Learning Transition Kernels of Jump-Diffusion Processes with Conditional Diffusion Models}

\section{Detailed Proofs}\label{sec:proofs}

\subsection{Proof of Section \ref{subsec:assumption}}

\begin{proof}[Proof of Lemma~\ref{lem:coefficient-increment-integrability}]
    It follows from \citet[Corollary 6.2.4]{MR2512800} that $\sup_{0\le t\le T}\mathbb E\|X_t\|^2<\infty.$
Then, by the definition of $\bar\mu_M$,
\[
M_{X,2}
=
\frac1M\sum_{m=0}^{M-1}\mathbb E\|X_{t_m}\|^2
\le
\sup_{0\le t\le T}\mathbb E\|X_t\|^2
<\infty.
\]

Moreover, the definition of $P_{\Delta}(x, \cdot)$ in \eqref{eq:transition-kernel-bis} and the
definition of $\bar\mu_M$ give
\begin{align*}
M_2
&=
\frac1M\sum_{m=0}^{M-1}
\mathbb E\|X_{t_{m+1}}-X_{t_m}\|^2
\le
\frac2M\sum_{m=0}^{M-1}
\left(
    \mathbb E\|X_{t_{m+1}}\|^2
    +\mathbb E\|X_{t_m}\|^2
\right)
\le
4\sup_{0\le t\le T}\mathbb E\|X_t\|^2
<\infty. \hfill\qedhere
\end{align*}
\end{proof}

\subsection{Proof of Section \ref{subsec:local-clip}}




\begin{proof}[Proof of Lemma~\ref{lem:tail-clipping-reduction}]
    Fix $s\in\mathcal F(L,\mathbf p,S,G,F)$ and
$k\in\{1,\ldots,K\}$, and set
\[
A_k:=\{(\tau_k,U_{\tau_k}^X,X)\in\mathsf D_R\}.
\]
In the following inequalities, all score functions are evaluated
at $(\tau_k,U_{\tau_k}^X,X)$.
On $A_k$,
\[
\|s-s^\star\|^2
\le2\|s-s_F^\star\|^2+2\|s_F^\star-s^\star\|^2.
\]
On $A_k^c$, the network constraint gives $\|s\|\le\sqrt d\,F$,
and hence
\[
\|s-s^\star\|^2\le2dF^2+2\|s^\star\|^2.
\]
Taking expectations with $X\sim\bar\mu_M$ and
$U_{\tau_k}^X\mid X\sim q_{\tau_k}(\cdot\mid X)$,
and averaging over $k$, proves
\eqref{eq:tail-clipping-reduction}.
\end{proof}

\subsection{Proofs of Section \ref{subsec:global-score-error}}

\subsubsection{Proof of Proposition~\ref{prop:localized-score-oracle}}

The proof of Proposition~\ref{prop:localized-score-oracle} depends on several auxiliary lemmas.
Throughout this subsection, we write $\mathcal F:=\mathcal F(L,\mathbf p,S,G,F).$ Let $\mathcal Z^\circ$ be an independent copy of
$\mathcal Z^{(1)}$, and independent of the training data $\mathcal D_{N}$.
For any measurable function
$g:(\mathbb R^d\times\mathbb R^d)^M\to\mathbb R$
such that $\mathbb E\bigl[|g(\mathcal Z^\circ)|\bigr]<\infty$, define
\[
\mathsf P g
:=
\mathbb E[g(\mathcal Z^\circ)],
\qquad
\mathsf P_N g
:=
\frac1N\sum_{n=1}^N g(\mathcal Z^{(n)}).
\]
Recall that by conditional Jensen's inequality, the conditional
expectation in \eqref{eq:clipped-population-score}
satisfies, 
\[
\|s_F^\star(\tau_k,u,x)\|\le F,
\qquad
k=1,\ldots,K,\quad u,x\in\mathbb R^d.
\]


\begin{lemma}
\label{lem:localized-loss-properties}
For every $s\in\mathcal F(L,\mathbf p,S,G,F)$, we have
\begin{equation}
\label{eq:localized-clipped-dsm-identity}
\begin{aligned}
\mathsf P \mathcal{G}_s^w
&=
\frac1K\sum_{k=1}^K\rho_{\tau_k}^2\,
\mathbb E\bigg[
\left\|
s(\tau_k,U_{\tau_k}^X,X)
-s_F^\star(\tau_k,U_{\tau_k}^X,X)
\right\|^2
\times
\mathbf 1_{\{(\tau_k,U_{\tau_k}^X,X)\in\mathsf D_R\}}
\bigg],
\end{aligned}
\end{equation}
where $X\sim\bar\mu_M$ and
$U_{\tau_k}^X\mid X\sim q_{\tau_k}(\cdot\mid X)$.
Consequently,
\begin{equation}
\label{eq:unweighted-weighted-localized-comparison}
\rho_{\min}^2\mathcal E_{R,F}(s)
\le \mathsf P \mathcal{G}_s^w
\le \mathcal E_{R,F}(s).
\end{equation}
Moreover, we have 
\begin{align}
\label{eq:clipped-envelope}
\|\mathcal{G}_s^w\|_\infty
\le 4dF^2,\quad 
\mathsf P\big[(\mathcal{G}_s^w)^2\big]
\le 8dF^2\,\mathsf P \mathcal{G}_s^w,
\end{align}
where $\|\mathcal{G}_s^w\|_\infty = \sup_{z\in(\mathbb R^d\times\mathbb R^d)^M}|\mathcal{G}_s^w(z)|$. 
\end{lemma}

\begin{proof}[Proof of Lemma \ref{lem:localized-loss-properties}]
Fix $s\in\mathcal F(L,\mathbf p,S,G,F)$.
Let $X\sim\bar\mu_M$, $Y\mid X\sim P_\Delta(X,\cdot)$,
and let $\xi\sim\mathcal N(0,I_d)$ be independent of $(X,Y)$.
For each fixed $k\in\{1,\ldots,K\}$, use the representation
\[
U_{\tau_k}^X=\alpha_{\tau_k}Y+\rho_{\tau_k}\xi,
\]
and define
\[
A_k:=\{(\tau_k,U_{\tau_k}^X,X)\in\mathsf D_R\},
\qquad
T_k:=
\left(1\wedge\frac{F\rho_{\tau_k}}{\|\xi\|}\right)
\left(-\frac{\xi}{\rho_{\tau_k}}\right),
\]
with $T_k:=0$ when $\xi=0$.
Then $\|T_k\|\le F$ and
\[
s_F^\star(\tau_k,U_{\tau_k}^X,X)
=\mathbb E[T_k\mid U_{\tau_k}^X,X].
\]
Set
\begin{align*}
a_k:=s(\tau_k,U_{\tau_k}^X,X)
      -s_F^\star(\tau_k,U_{\tau_k}^X,X),\quad 
r_k:=s_F^\star(\tau_k,U_{\tau_k}^X,X)-T_k,
\end{align*}
and
\begin{align*}
D_k
&:=\|s(\tau_k,U_{\tau_k}^X,X)-T_k\|^2
   -\|s_F^\star(\tau_k,U_{\tau_k}^X,X)-T_k\|^2
=\|a_k\|^2+2\langle a_k,r_k\rangle.
\end{align*}
Since $\mathbb E[r_k\mid U_{\tau_k}^X,X]=0$,
\[
\mathbb E[D_k\mid U_{\tau_k}^X,X]=\|a_k\|^2.
\]
Since $A_k$ is measurable with respect to $(U_{\tau_k}^X,X)$,
the definition of $\mathcal{G}_s^w$ and
$\bar\mu_M=M^{-1}\sum_{m=0}^{M-1}\mu_m$ give
\[
\mathsf P \mathcal{G}_s^w
=\frac1K\sum_{k=1}^K\rho_{\tau_k}^2
  \mathbb E[D_k\mathbf 1_{A_k}]
=\frac1K\sum_{k=1}^K\rho_{\tau_k}^2
  \mathbb E[\|a_k\|^2\mathbf 1_{A_k}].
\]
This proves \eqref{eq:localized-clipped-dsm-identity}.
The comparison \eqref{eq:unweighted-weighted-localized-comparison}
follows from $\rho_{\min}^2\le\rho_{\tau_k}^2\le1$.

The network constraint gives $\|s(\tau_k,u,x)\|\le\sqrt d\,F$.
Together with $\|s_F^\star(\tau_k,u,x)\|\le F$ and
$\|T_k\|\le F$, this implies that both squared norms defining
$D_k$ lie in $[0,4dF^2]$. Hence $|D_k|\le4dF^2$.
The same pointwise bound holds for every input pair $(x,y)$. Taking expectation with respect to the Gaussian noise
and averaging over $m$ and $k$ therefore yields the first inequality in 
\eqref{eq:clipped-envelope}.

For the second-moment bound, observe that
\begin{align*}
\mathbb E[\|r_k\|^2\mid U_{\tau_k}^X,X]
&=\mathbb E[\|T_k\|^2\mid U_{\tau_k}^X,X]
  -\|s_F^\star(\tau_k,U_{\tau_k}^X,X)\|^2
\le F^2.
\end{align*}
Since $a_k$ is measurable with respect to $(U_{\tau_k}^X,X)$ and
$\|a_k\|\le(\sqrt d+1)F\le2\sqrt d\,F$,
\begin{align*}
\mathbb E[D_k^2\mid U_{\tau_k}^X,X]
&=\|a_k\|^4
  +4\mathbb E[\langle a_k,r_k\rangle^2\mid U_{\tau_k}^X,X]\le\|a_k\|^4+4F^2\|a_k\|^2
\le8dF^2\|a_k\|^2.
\end{align*}
By Jensen's inequality and the definition of $\bar\mu_M$,
we obtain
\begin{align*}
\mathsf P\big[(\mathcal{G}_s^w)^2\big]
&\le\frac1K\sum_{k=1}^K\rho_{\tau_k}^4
    \mathbb E[D_k^2\mathbf 1_{A_k}]
\le8dF^2\frac1K\sum_{k=1}^K\rho_{\tau_k}^2
    \mathbb E[\|a_k\|^2\mathbf 1_{A_k}]
=8dF^2\,\mathsf P \mathcal{G}_s^w,
\end{align*}
where we used $\rho_{\tau_k}^4\le\rho_{\tau_k}^2$, which concludes the proof.
\end{proof}

\begin{lemma}
\label{lem:block-loss-entropy}
There exists a universal constant $C>0$ such that, for every
$0<\epsilon\le1$,
\begin{equation}
\label{eq:block-loss-entropy}
\begin{aligned}
&\log\mathcal N\!\left(
\epsilon,
\{\mathcal{G}_s^w:s\in\mathcal F(L,\mathbf p,S,G,F)\},
\|\cdot\|_\infty
\right)
\\
&\quad\le
C S_{\mathrm{tot}}(L+1)
\log\!\left(
\frac{
C d\left[
F^2+F(1+R)(1\vee G)(L+1)
\left(1+\max_{0\le\ell\le L+1}p_\ell\right)
\right]
}{\epsilon}
\right).
\end{aligned}
\end{equation}
Here $\mathcal N(\epsilon,\mathcal H,\|\cdot\|_\infty)$
denotes the minimum cardinality of an $\epsilon$-net of
$\mathcal H$, with centres in $\mathcal H$, under the supremum norm
on $(\mathbb R^d\times\mathbb R^d)^M$.
\end{lemma}

\begin{proof}[Proof of Lemma \ref{lem:block-loss-entropy}]
To simplify notation, we write $\mathcal G
:=
\{\mathcal{G}_s^w:s\in\mathcal F(L,\mathbf p,S,G,F)\}$.
For a function $f:\mathcal X\to\mathbb R^d$, define
\[
\|f\|_{\infty,R,K}
:=
\max_{1\le k\le K}
\sup_{(u,x)\in\mathcal U_R\times\mathcal X_R}
\|f(\tau_k,u,x)\|_\infty.
\]
For $s,\widetilde s\in\mathcal F(L,\mathbf p,S,G,F)$ and
any $t\in\mathbb R^d$ with $\|t\|\le F$,
\begin{align*}
\big|\|s(\tau_k,u,x)-t\|^2
      -\|\widetilde s(\tau_k,u,x)-t\|^2\big|
&\le
\|s(\tau_k,u,x)-\widetilde s(\tau_k,u,x)\|
\times
\big(\|s(\tau_k,u,x)\|
     +\|\widetilde s(\tau_k,u,x)\|+2\|t\|\big)
\\
&\le
4dF\,
\|s(\tau_k,u,x)-\widetilde s(\tau_k,u,x)\|_\infty.
\end{align*}
The terms involving $s_F^\star$ cancel in
$\mathcal{G}_s^w-\mathcal{G}_{\widetilde s}^w$.
Since $\rho_{\tau_k}^2\le1$ and the loss is restricted to
$\mathsf D_R$, it follows that
\begin{equation}
\label{eq:block-loss-lipschitz}
\|\mathcal{G}_s^w-\mathcal{G}_{\widetilde s}^w\|_\infty
\le4dF\,\|s-\widetilde s\|_{\infty,R,K}.
\end{equation}

Set
\[
p_*:=\max_{0\le\ell\le L+1}p_\ell,
\qquad G_*:=1\vee G,
\qquad B:=G_*(1+p_*).
\]
Let $\theta$ denote the network parameters and
$h_{\ell,\theta}^{(k)}$ the corresponding hidden-layer outputs.
The parameter bounds and the $1$-Lipschitz property of ReLU imply
\[
\max_{1\le k\le K}
\sup_{(u,x)\in\mathcal U_R\times\mathcal X_R}
\|h_{\ell,\theta}^{(k)}(u,x)\|_\infty
\le(1+R)B^\ell,
\qquad \ell=0,\ldots,L,
\]
uniformly over admissible $\theta$.
For two admissible parameter vectors $\theta,\widetilde\theta$,
put $r:=\|\theta-\widetilde\theta\|_\infty$ and
\[
e_\ell:=
\max_{1\le k\le K}
\sup_{(u,x)\in\mathcal U_R\times\mathcal X_R}
\|h_{\ell,\theta}^{(k)}(u,x)
 -h_{\ell,\widetilde\theta}^{(k)}(u,x)\|_\infty.
\]
Then $e_0=0$ and
\[
e_{\ell+1}
\le B e_\ell+2(1+R)B^{\ell+1}r,
\qquad \ell=0,\ldots,L-1.
\]
The same bound applies to the output layer.
Iterating the recurrence gives
\begin{equation}
\label{eq:score-parameter-perturbation}
\|s_\theta-s_{\widetilde\theta}\|_{\infty,R,K}
\le
2(1+R)(L+1)B^{L+1}
\|\theta-\widetilde\theta\|_\infty.
\end{equation}

There are
\[
D:=\sum_{\ell=0}^{L}p_{\ell+1}(p_\ell+1)
\le(L+1)p_*(p_*+1)
\]
possible shared weight and bias positions.
The number of supports of size at most $S$ is at most $(D+1)^S$.
For any fixed support, there are at most $S_{\mathrm{tot}}$
free coordinates, all lying in $[-G_*,G_*]$.
For each free coordinate, partition $[-G_*,G_*]$
into intervals of length at most $\delta$, where
\[
\delta:=
\frac{\epsilon}
{8dF(1+R)(L+1)B^{L+1}}.
\]
For each fixed support, choose one admissible parameter vector
from each cell containing such a vector.
The selected vectors form a $\delta$-net, since
$\|\theta-\widetilde\theta\|_\infty\le\delta$
whenever $\theta$ and $\widetilde\theta$ lie in the same cell.
The number of selected vectors is at most
\[
\left(1+\frac{2G_*}{\delta}\right)^{S_{\mathrm{tot}}}.
\]
By \eqref{eq:block-loss-lipschitz} and
\eqref{eq:score-parameter-perturbation}, the selected
parameters yield an $\epsilon$-net of $\mathcal G$.
Taking the union of these nets over all supports, we obtain
\begin{align*}
\log\mathcal N(\epsilon,\mathcal G,\|\cdot\|_\infty)
&\le S\log(D+1)
 +S_{\mathrm{tot}}\log\!\left(1+\frac{2G_*}{\delta}\right)
\\
&=S\log(D+1)
 +S_{\mathrm{tot}}\log\!\left(
 1+\frac{16dF(1+R)(L+1)G_*B^{L+1}}{\epsilon}
 \right).
\end{align*}
Since $S\le S_{\mathrm{tot}}$,
$D\le(L+1)p_*(p_*+1)$,
$B=G_*(1+p_*)$, $F\ge1$, and $0<\epsilon\le1$,
the preceding bound implies
\eqref{eq:block-loss-entropy}
for a sufficiently large universal constant $C>0$.
\end{proof}

\begin{lemma}
\label{lem:localized-relative-deviation}
Let $\mathcal G$ be a nonempty class of measurable real-valued
functions on $(\mathbb R^d\times\mathbb R^d)^M$.
Suppose that there exist constants $1\le A\le B$, $V\ge1$,
and $\mathfrak C\ge e$ such that, for every $g\in\mathcal G$,
\[
\mathsf P g\ge0,
\qquad
\|g\|_\infty\le A,
\qquad
\mathsf P(g^2)\le B\,\mathsf P g,
\]
and
\[
\log\mathcal N(\epsilon,\mathcal G,\|\cdot\|_\infty)
\le V\log\left(\frac{\mathfrak C}{\epsilon}\right),
\qquad 0<\epsilon\le1.
\]
Then there exists a universal constant $C>0$ such that
\begin{equation}
\label{eq:generic-localized-relative-deviation}
\mathbb E\left[
\sup_{g\in\mathcal G}
\left\{\frac12\mathsf P g-\mathsf P_N g\right\}
\right]
\le
C B\,\frac{V\log(\mathfrak C N)}{N}.
\end{equation}
\end{lemma}

\begin{proof}[Proof of Lemma \ref{lem:localized-relative-deviation}]
For each integer $j\ge1$, choose a finite $1/j$-net
$\mathcal G_j\subseteq\mathcal G$ under $\|\cdot\|_\infty$.
For any $g,h\in\mathcal G$,
\[
\left|
\left(\frac12\mathsf P g-\mathsf P_N g\right)
-
\left(\frac12\mathsf P h-\mathsf P_N h\right)
\right|
\le \frac32\|g-h\|_\infty.
\]
Consequently,
\[
0\le
\sup_{g\in\mathcal G}
\left\{\frac12\mathsf P g-\mathsf P_N g\right\}
-
\max_{g\in\mathcal G_j}
\left\{\frac12\mathsf P g-\mathsf P_N g\right\}
\le \frac{3}{2j}.
\]
Since $\mathcal G_j$ is finite, the maximum over
$\mathcal G_j$ is measurable.
Letting $j\to\infty$, we conclude that the supremum over
$\mathcal G$ is measurable.

Set $\eta:=1/N$ and choose an $\eta$-net
$\{g_1,\ldots,g_J\}\subseteq\mathcal G$ satisfying
\[
\log J\le V\log(\mathfrak C N).
\]
For every $g\in\mathcal G$, choose $g_j$ with
$\|g-g_j\|_\infty\le\eta$. Then
\[
\frac12\mathsf P g-\mathsf P_N g
\le
\frac12\mathsf P g_j-\mathsf P_N g_j
+\frac{3}{2N},
\]
and hence
\begin{equation}
\label{eq:relative-deviation-net-reduction}
\sup_{g\in\mathcal G}
\left\{\frac12\mathsf P g-\mathsf P_N g\right\}
\le
\max_{1\le j\le J}
\left\{\frac12\mathsf P g_j-\mathsf P_N g_j\right\}
+\frac{3}{2N}.
\end{equation}

Fix $j\in\{1,\ldots,J\}$.
The random variables
$\mathsf P g_j-g_j(\mathcal Z^{(n)})$, $n=1,\ldots,N$,
are independent and centred, and satisfy
\[
\left|\mathsf P g_j-g_j(\mathcal Z^{(n)})\right|
\le 2A
\]
and
\[
\operatorname{Var}\!\left(
\mathsf P g_j-g_j(\mathcal Z^{(n)})
\right)
\le \mathsf P(g_j^2)
\le B\,\mathsf P g_j.
\]
Bernstein's inequality therefore gives, for $t>0$,
\begin{align*}
\mathbb P\left(
\frac12\mathsf P g_j-\mathsf P_N g_j\ge t
\right)
&\le
\exp\left(
-\frac{N(\frac12\mathsf P g_j+t)^2}
{2B\mathsf P g_j+\frac{4A}{3}
(\frac12\mathsf P g_j+t)}
\right)
\le
\exp\left(-\frac{3Nt}{8B}\right).
\end{align*}
For the last inequality, we used $A\le B$, $\mathsf P g_j\ge0$,
and
\[
\left(\frac12\mathsf P g_j+t\right)^2
\ge t(\mathsf P g_j+t),
\]
together with
\[
2B\mathsf P g_j+\frac{4A}{3}
\left(\frac12\mathsf P g_j+t\right)
\le\frac{8B}{3}(\mathsf P g_j+t).
\]
The union bound yields
\[
\mathbb P\left(
\max_{1\le j\le J}
\left\{\frac12\mathsf P g_j-\mathsf P_N g_j\right\}
\ge t
\right)
\le
\min\left\{1,J\exp\left(-\frac{3Nt}{8B}\right)\right\}.
\]
Integrating this tail bound over $t>0$ gives
\begin{align*}
\mathbb E\left[
\max_{1\le j\le J}
\left\{\frac12\mathsf P g_j-\mathsf P_N g_j\right\}
\right]
&\le
\int_0^\infty
\min\left\{1,J e^{-3Nt/(8B)}\right\}\,dt
=\frac{8B}{3N}(1+\log J).
\end{align*}
Combining this estimate with
\eqref{eq:relative-deviation-net-reduction}, we obtain
\[
\mathbb E\left[
\sup_{g\in\mathcal G}
\left\{\frac12\mathsf P g-\mathsf P_N g\right\}
\right]
\le
\frac{8B}{3N}(1+\log J)+\frac{3}{2N}.
\]
Since $\log J\le V\log(\mathfrak C N)$,
$B\ge1$, $V\ge1$, and $\log(\mathfrak C N)\ge1$,
this proves \eqref{eq:generic-localized-relative-deviation}
for a universal constant $C>0$.
\end{proof}

\begin{proof}[Proof of Proposition~\ref{prop:localized-score-oracle}]
Let
\[
\mathcal G
:=
\{\mathcal{G}_s^w:s\in\mathcal F(L,\mathbf p,S,G,F)\}.
\]
By Lemmas~\ref{lem:localized-loss-properties}
and~\ref{lem:block-loss-entropy},
Lemma~\ref{lem:localized-relative-deviation} applies to
$\mathcal G$ with $A=4dF^2$, $B=8dF^2$, and the entropy
bound in \eqref{eq:block-loss-entropy}. Consequently,
\begin{equation}
\label{eq:score-class-relative-deviation}
\begin{aligned}
&\mathbb E_{\mathcal D_N}\left[
\sup_{g\in\mathcal G}
\left\{\frac12\mathsf P g-\mathsf P_N g\right\}
\right]
\le
C dF^2\frac{S_{\mathrm{tot}}(L+1)}{N}
\log\!\left(
C N d\left[
F^2+F(1+R)(1\vee G)(L+1)
\left(1+\max_{0\le\ell\le L+1}p_\ell\right)
\right]
\right)
\end{aligned}
\end{equation}
for a universal constant $C>0$.

Since $s_{\hat\theta_N}\in\mathcal F(L,\mathbf p,S,G,F)$,
for every realization of the training data,
\[
\mathsf P \mathcal{G}_{s_{\hat\theta_N}}^w
\le
2\sup_{g\in\mathcal G}
\left\{\frac12\mathsf P g-\mathsf P_N g\right\}
+2\mathsf P_N \mathcal{G}_{s_{\hat\theta_N}}^w.
\]
Fix a deterministic $s\in\mathcal F(L,\mathbf p,S,G,F)$.
By \eqref{eq:expected-localized-optimization-error},
\begin{align*}
\mathbb E_{\mathcal D_N}
\left[\mathsf P_N \mathcal{G}_{s_{\hat\theta_N}}^w\right]
&=
\Psi_{N,R,F}^w(s_{\hat\theta_N})
+\mathbb E_{\mathcal D_N}\left[
\inf_{\widetilde s\in\mathcal F(L,\mathbf p,S,G,F)}
\mathsf P_N \mathcal{G}_{\widetilde s}^w
\right]
\\
&\le
\Psi_{N,R,F}^w(s_{\hat\theta_N})
+\mathbb E_{\mathcal D_N}\left[\mathsf P_N \mathcal{G}_s^w\right]
=
\Psi_{N,R,F}^w(s_{\hat\theta_N})+\mathsf P \mathcal{G}_s^w.
\end{align*}
Taking expectations yields
\begin{align*}
\mathbb E_{\mathcal D_N}
\left[\mathsf P \mathcal{G}_{s_{\hat\theta_N}}^w\right]
&\le
2\mathbb E_{\mathcal D_N}\left[
\sup_{g\in\mathcal G}
\left\{\frac12\mathsf P g-\mathsf P_N g\right\}
\right]
+2\mathsf P \mathcal{G}_s^w
+2\Psi_{N,R,F}^w(s_{\hat\theta_N}).
\end{align*}
Applying \eqref{eq:unweighted-weighted-localized-comparison}
and taking the infimum over deterministic
$s\in\mathcal F(L,\mathbf p,S,G,F)$, we obtain
\begin{align*}
\mathbb E_{\mathcal D_N}
\left[\mathcal E_{R,F}(s_{\hat\theta_N})\right]
&\le
\frac{2}{\rho_{\min}^2}\bigg[
\inf_{s\in\mathcal F(L,\mathbf p,S,G,F)}
\mathcal E_{R,F}(s)
+\Psi_{N,R,F}^w(s_{\hat\theta_N})
+\mathbb E_{\mathcal D_N}\left[
\sup_{g\in\mathcal G}
\left\{\frac12\mathsf P g-\mathsf P_N g\right\}
\right]
\bigg].
\end{align*}
Substituting \eqref{eq:score-class-relative-deviation}
yields the desired bound.
\end{proof}


\subsubsection{Proofs of Lemma~\ref{lem:clipping-tail-control}, Theorem~\ref{thm:score-error-bound}
and Corollary~\ref{cor:score-error-rate}}

\begin{proof}[Proof of Lemma~\ref{lem:clipping-tail-control}]
Let $X\sim\bar\mu_M$, $Y\mid X\sim P_\Delta(X,\cdot)$,
and let $\xi\sim\mathcal N(0,I_d)$ be independent of $(X,Y)$.
By Lemma~\ref{lem:coefficient-increment-integrability},
\[
\mathbb E\|X\|^2=M_{X,2}<\infty,
\qquad
\mathbb E\|Y\|^2=M_2<\infty.
\]
For each fixed $k\in\{1,\ldots,K\}$, use the representation
\[
U_{\tau_k}^X=\alpha_{\tau_k}Y+\rho_{\tau_k}\xi,
\]
and set
\[
A_k:=\{(\tau_k,U_{\tau_k}^X,X)\in\mathsf D_R\},
\qquad
T_k:=
\left(1\wedge\frac{F\rho_{\tau_k}}{\|\xi\|}\right)
\left(-\frac{\xi}{\rho_{\tau_k}}\right),
\]
with $T_k:=0$ when $\xi=0$.
By \eqref{eq:conditional-score-identity}
and \eqref{eq:clipped-population-score},
\[
s^\star(\tau_k,U_{\tau_k}^X,X)
=\mathbb E\left[-\frac{\xi}{\rho_{\tau_k}}
\,\middle|\,U_{\tau_k}^X,X\right],
\qquad
s_F^\star(\tau_k,U_{\tau_k}^X,X)
=\mathbb E[T_k\mid U_{\tau_k}^X,X].
\]
Since $A_k$ is measurable with respect to $(U_{\tau_k}^X,X)$,
conditional Jensen's inequality gives
\begin{align*}
&\mathbb E\left[
\|s_F^\star(\tau_k,U_{\tau_k}^X,X)
 -s^\star(\tau_k,U_{\tau_k}^X,X)\|^2\mathbf 1_{A_k}
\right]
\le
\mathbb E\left[
\left\|T_k+\frac{\xi}{\rho_{\tau_k}}\right\|^2
\mathbf 1_{A_k}
\right]
\le
\rho_{\tau_k}^{-2}\mathbb E\left[
\|\xi\|^2\mathbf 1_{\{\|\xi\|>F\rho_{\tau_k}\}}
\right].
\end{align*}
Averaging over $k$ and using $\rho_{\tau_k}\ge\rho_{\min}$,
we obtain
\begin{equation}
\label{eq:clipping-reduction-gaussian}
\mathcal C_{R,F}
\le
\rho_{\min}^{-2}\mathbb E\left[
\|\xi\|^2\mathbf 1_{\{\|\xi\|>F\rho_{\min}\}}
\right].
\end{equation}

For $0<\lambda<1/2$,
\[
\mathbb E e^{\lambda\|\xi\|^2}=(1-2\lambda)^{-d/2}.
\]
For $r>\sqrt d$, exponential Markov's inequality with
$\lambda=(1-d/r^2)/2$ yields
\[
\mathbb P(\|\xi\|>r)
\le
\exp\left(-\frac{r^2-d}{2}+d\log\frac{r}{\sqrt d}\right)
\le
\exp\left(-\frac12(r-\sqrt d)^2\right),
\]
where the last inequality follows from $\log x\le x-1$.
Using $\mathbb P(\|\xi\|\ge r)\le1$ for
$0\le r\le\sqrt d$, we conclude that
\[
\mathbb P(\|\xi\|>r)
\le
\exp\left(-\frac12\max\{r-\sqrt d,0\}^2\right),
\qquad r\ge0.
\]
Since $\mathbb E\|\xi\|^4=d(d+2)\le3d^2$,
the Cauchy--Schwarz inequality gives
\[
\mathbb E\left[\|\xi\|^2\mathbf 1_{\{\|\xi\|>r\}}\right]
\le
\sqrt{3}\,d\,
\exp\left(-\frac14\max\{r-\sqrt d,0\}^2\right).
\]
Substituting $r=F\rho_{\min}$ into
\eqref{eq:clipping-reduction-gaussian} gives
\[
\mathcal C_{R,F}
\le
\sqrt{3}\,\rho_{\min}^{-2}d\,
\exp\left(-\frac14\max\{F\rho_{\min}-\sqrt d,0\}^2\right).
\]

We next bound $\mathcal T_{R,F}$.
Independence of $\xi$ and $(X,Y)$ gives
\[
\mathbb E\|U_{\tau_k}^X\|^2
=\alpha_{\tau_k}^2M_2+\rho_{\tau_k}^2d
\le M_2+d.
\]
Since $\tau_k\in[\tau_1,\mathcal T]$,
\[
A_k^c
=\{\|X\|>R\}\cup\{\|U_{\tau_k}^X\|>R\}.
\]
Markov's inequality therefore yields
\begin{equation}
\label{eq:outside-domain-probability}
\mathbb P(A_k^c)
\le
\frac{\mathbb E\|X\|^2+\mathbb E\|U_{\tau_k}^X\|^2}{R^2}
\le
(M_{X,2}+M_2+d)R^{-2}.
\end{equation}
Moreover, conditional Jensen's inequality gives
\[
\mathbb E\|s^\star(\tau_k,U_{\tau_k}^X,X)\|^4
\le
\rho_{\tau_k}^{-4}\mathbb E\|\xi\|^4
\le3\rho_{\min}^{-4}d^2.
\]
Combining this estimate with
\eqref{eq:outside-domain-probability} and the
Cauchy--Schwarz inequality, we obtain
\begin{align*}
&\mathbb E\left[
\|s^\star(\tau_k,U_{\tau_k}^X,X)\|^2\mathbf 1_{A_k^c}
\right]
\le
\left(\mathbb E\|s^\star(\tau_k,U_{\tau_k}^X,X)\|^4\right)^{1/2}
\mathbb P(A_k^c)^{1/2}
\le
\sqrt{3}\,\rho_{\min}^{-2}d
(M_{X,2}+M_2+d)^{1/2}R^{-1}.
\end{align*}
By \eqref{eq:tail-term}, averaging over $k$ yields
\begin{align*}
\mathcal T_{R,F}
&\le
2dF^2(M_{X,2}+M_2+d)R^{-2}
+
2\sqrt{3}\,\rho_{\min}^{-2}d
(M_{X,2}+M_2+d)^{1/2}R^{-1}.
\end{align*}
This proves \eqref{eq:tail-upper-bound}.
\end{proof}


\begin{proof}[Proof of Theorem~\ref{thm:score-error-bound}]
Applying Lemma~\ref{lem:tail-clipping-reduction} to
$s_{\hat\theta_N}$ and taking expectations gives
\[
\mathbb E_{\mathcal D_N}\left[\mathcal E(s_{\hat\theta_N})\right]
\le
2\mathbb E_{\mathcal D_N}
\left[\mathcal E_{R,F}(s_{\hat\theta_N})\right]
+2\mathcal C_{R,F}+\mathcal T_{R,F}.
\]
Proposition~\ref{prop:localized-score-oracle},
Lemma~\ref{lem:clipping-tail-control}, and the definition of
$\mathfrak S_{N,R,F}$ now imply
\eqref{eq:score-error-bound}.
\end{proof}

\begin{proof}[Proof of Corollary~\ref{cor:score-error-rate}]
Throughout this proof, $C>0$ denotes a constant depending only on
$\rho_{\min}$, $L$, $G$, and $C_{\mathrm{width}}$ and
its value may change from line to line.
The definition of $F_N$ gives
\[
F_N\rho_{\min}-\sqrt d=2\sqrt{\log N},
\qquad
F_N^2\le C(d+\log N).
\]

Then
\[
\rho_{\min}^{-2}d
\exp\left(
-\frac14\max\{F_N\rho_{\min}-\sqrt d,0\}^2
\right)
\le C\frac{d}{N}.
\]
Since $R_N=N$,
\[
dF_N^2(M_{X,2}+M_2+d)R_N^{-2}
\le
C\frac{d(d+\log N)}{N^2}(M_{X,2}+M_2+d).
\]
and 
\begin{align*}
\rho_{\min}^{-2}d(M_{X,2}+M_2+d)^{1/2}R_N^{-1}
&\le C\frac dN(M_{X,2}+M_2+d)^{1/2}.
\end{align*}
If $\sqrt{M_{X,2}+M_2+d}\le N$, then, since $(M_{X,2}+M_2+d)\ge d\ge1$, the clipping and tail terms together are
bounded by
\[
C\frac dN\left[\sqrt{M_{X,2}+M_2+d}
+\frac{(d+\log N)(M_{X,2}+M_2+d)}{N}
\right]
\le C\frac{d(d+\log N)\sqrt {M_{X,2}+M_2+d}}{N}.
\]
If instead $\sqrt{M_{X,2}+M_2+d}>N$,
the conditional score identity and conditional Jensen's
inequality give
\[
\mathbb E\|s^\star(\tau_k,U_{\tau_k}^X,X)\|^2
\le \rho_{\tau_k}^{-2}d
\le \rho_{\min}^{-2}d,
\qquad k=1,\ldots,K,
\]
where $X\sim\bar\mu_M$ and
$U_{\tau_k}^X\mid X\sim q_{\tau_k}(\cdot\mid X)$.
Together with the network constraint
$\|s_{\hat\theta_N}(\tau_k,u,x)\|^2\le dF_N^2$, this yields 
\[
\mathbb E_{\mathcal D_N}
\left[\mathcal E(s_{\hat\theta_N})\right]
\le 2dF_N^2+2\rho_{\min}^{-2}d
\le C d(d+\log N)
\le C\frac{d(d+\log N)\sqrt {M_{X,2}+M_2+d}}{N}. 
\]
Using
$\max_{0\le\ell\le L+1}p_\ell\le C_{\mathrm{width}}d$,
together with $R_N=N$ and $F_N^2\le C(d+\log N)$,
we obtain
\begin{align*}
&CNd\left[
F_N^2+F_N(1+R_N)(1\vee G)(L+1)
\left(1+\max_{0\le\ell\le L+1}p_\ell\right)
\right]\le C(N+d)^{9/2}.
\end{align*}
Since $N+d\ge3$, it follows that
\begin{align*}
&\log\!\left(
CNd\left[
F_N^2+F_N(1+R_N)(1\vee G)(L+1)
\left(1+\max_{0\le\ell\le L+1}p_\ell\right)
\right]
\right)\le C\log(N+d).
\end{align*}
Therefore, the estimation term in $\mathfrak S_{N,R_N,F_N}$
is bounded by
\[
C\frac{
dS_{\mathrm{tot}}(L+1)(d+\log N)\log(N+d)
}{N}.
\]
Substituting these bounds into
Theorem~\ref{thm:score-error-bound}
proves the assertion.
\end{proof}

\subsection{Proof of Section \ref{subsec:KL-gene-path}}

Recall that the transition kernel $\bar P_{\Delta}(x,\cdot)$ defined by \eqref{eq:transition-kernel-def} and the increment kernel $P_{\Delta}(x,\cdot)$ defined by \eqref{eq:transition-kernel-bis} are related by
\[
\bar P_{\Delta}(x,B)=P_{\Delta}(x,B-x),
\qquad \text{for every } B\in\mathcal B(\mathbb R^d) \text{ and where } B-x:=\{y\in\mathbb R^d:x+y\in B\}.
\]
Similarly, using the estimated increment kernel $\widehat Q_{N,\Delta}(\cdot\mid x)$ defined in \eqref{eq:generated-increment-law}, we define the estimated transition kernel by
\[
\widehat P_{N,\Delta}(x,B):=\widehat Q_{N,\Delta}(B-x\mid x),\qquad B\in\mathcal B(\mathbb R^d).
\]
By the Markov property of \(X\) and the recursive construction of the generated path defined by \eqref{eq:construction-new-path}, the laws of the true discretely observed path \((X_{t_0},\ldots,X_{t_M})\) and the generated path \((\widehat X_{t_0},\ldots,\widehat X_{t_M})\), denoted  respectively by \(P_{0:M}\) and \(\widehat P_{0:M}^{\,N}\), can be written as
\begin{align}
\label{eq:true-path-law}
    P_{0:M}(dz_0,\ldots,dz_M)
    &=
    \mu_0(dz_0)
    \prod_{m=0}^{M-1}
    \bar P_\Delta(z_m,dz_{m+1}),
    \\
\label{eq:generated-path-law}
    \widehat{P}_{0:M}^{\,N}(dz_0,\ldots,dz_M)
    &=
    \mu_0(dz_0)
    \prod_{m=0}^{M-1}
    \widehat{P}_{N,\Delta}(z_m,dz_{m+1}).
\end{align}






\begin{lemma}
\label{lem:path-kl-decomposition}
Conditionally on $\mathcal D_N$,
\begin{equation}
\label{eq:path-kl-decomposition}
\mathrm{KL}\left(
P_{0:M}\middle\|\widehat P_{0:M}^{\,N}
\right)
=
\sum_{m=0}^{M-1}
\int_{\mathbb R^d}
\mathrm{KL}\left(
P_\Delta(x,\cdot)
\middle\|
\widehat Q_{N,\Delta}(\cdot\mid x)
\right)\,\mu_m(dx).
\end{equation}
\end{lemma}

\begin{proof}[Proof of Lemma \ref{lem:path-kl-decomposition}]
Fix a realization of $\mathcal D_N$.
By \eqref{eq:true-path-law} and \eqref{eq:generated-path-law},
the two path laws have the same initial distribution $\mu_0$.
The chain rule for KL divergence, together with the invariance of
KL divergence under the translation $y\mapsto x+y$, implies 
\[
\mathrm{KL}\left(
P_{0:M}\middle\|\widehat P_{0:M}^{\,N}
\right)=\sum_{m=0}^{M-1}
\int_{\mathbb R^d}
\mathrm{KL}\left(
\bar P_\Delta(x,\cdot)
\middle\|
\widehat P_{N,\Delta}(x,\cdot)
\right)\,\mu_m(dx)
=
\sum_{m=0}^{M-1}
\int_{\mathbb R^d}
\mathrm{KL}\left(
P_\Delta(x,\cdot)
\middle\|
\widehat Q_{N,\Delta}(\cdot|x)
\right)\,\mu_m(dx).\qedhere
\]
\end{proof}

\begin{proof}[Proof of Theorem~\ref{thm:path-kl-bound}]
Throughout this proof, $C>0$ denotes a constant depending only on
the fixed parameters $a_0$, $b_0$, and $\mathcal T$ and
its value may change from line to line.

Fix a realization of $\mathcal D_N$.
We apply the estimate proved in
\citet[Proposition~2]{gao2025generating}
with conditioning variable $X\sim\bar\mu_M$, target law
$P_\Delta(X,\cdot)$, terminal time $\mathcal T$, step size $h$,
and score estimator $s_{\hat\theta_N}$.
The sampler in \eqref{eq:reverse-euler-sampler}
coincides with the Euler scheme used in \citet[Proposition~2]{gao2025generating}.
Assumption~\ref{ass:increment-density-regularity} supplies the
density and Lipschitz-score conditions for $\bar\mu_M$-almost every $x\in\mathbb R^d$, and
Lemma~\ref{lem:coefficient-increment-integrability} gives $M_2<\infty$.
The score error in the cited estimate is
$\mathcal E(s_{\hat\theta_N})$.
By \eqref{eq:conditional-score-identity},
conditional Jensen's inequality, and the network constraint,
we have
\[
\mathcal E(s_{\hat\theta_N})
\le 2dF^2+2\rho_{\min}^{-2}d<\infty.
\]
Applying the cited estimate, we obtain
\begin{align*}
&\mathbb E_{X\sim\bar\mu_M}\left[
\mathrm{KL}\left(
P_\Delta(X,\cdot)
\middle\|
\widehat Q_{N,\Delta}(\cdot\mid X)
\right)\right]
\\
&\quad\le C\bigg\{
\alpha_{\mathcal T}^2
\left(M_2+d\alpha_{\mathcal T}^2\right)
+(\mathcal T+1)^3\mathcal E(s_{\hat\theta_N})
+h(M_2+d)(\mathcal T+1)^5
+M_2h(\mathcal T+1)^2
\\
&\qquad\qquad
+dh\left[L_1+(\mathcal T+1)^4\right]
+dh^2(L_1+\mathcal T+1)
\bigg\}.
\end{align*}
Using
\[
\alpha_{\mathcal T}^2
=\exp\left(-\frac{a_0}{2}\mathcal T^2-b_0\mathcal T\right)
\le1,
\qquad
h=\mathcal T/K\le\mathcal T,
\]
and, in particular,
\[
dh^2(L_1+\mathcal T+1)
\le dh(\mathcal T+1)^2(1+L_1),
\]
we obtain
\begin{equation}
\begin{aligned}
&\mathbb E_{X\sim\bar\mu_M}\left[
\mathrm{KL}\left(
P_\Delta(X,\cdot)
\middle\|
\widehat Q_{N,\Delta}(\cdot\mid X)
\right)\right]
\le
C\left\{
\eta_{\mathcal T,h}\left[M_2+d(1+L_1)\right]
+(\mathcal T+1)^3\mathcal E(s_{\hat\theta_N})
\right\}. \nonumber
\end{aligned}
\end{equation}

By Lemma~\ref{lem:path-kl-decomposition}, Tonelli's theorem,
and $\bar\mu_M=M^{-1}\sum_{m=0}^{M-1}\mu_m$,
\begin{align*}
\mathcal R_N
&=
\sum_{m=0}^{M-1}
\mathbb E_{\mathcal D_N}\left[
\int_{\mathbb R^d}
\mathrm{KL}\left(
P_\Delta(x,\cdot)
\middle\|
\widehat Q_{N,\Delta}(\cdot\mid x)
\right)\,\mu_m(dx)
\right]
\\
&=
M\mathbb E_{\mathcal D_N}
\mathbb E_{X\sim\bar\mu_M}\left[
\mathrm{KL}\left(
P_\Delta(X,\cdot)
\middle\|
\widehat Q_{N,\Delta}(\cdot\mid X)
\right)\right].
\end{align*}
Taking expectations in \eqref{eq:one-step-kl-bound} therefore gives
\[
\mathcal R_N
\le CM\left\{
\eta_{\mathcal T,h}\left[M_2+d(1+L_1)\right]
+(\mathcal T+1)^3
\mathbb E_{\mathcal D_N}\left[\mathcal E(s_{\hat\theta_N})\right]
\right\}.
\]
Applying Theorem~\ref{thm:score-error-bound} concludes the proof.
\end{proof}


\section{Additional Experimental Details}

\subsection{Additional Details for Section \ref{subsec:numerical-cvg}}

\paragraph{Simulation of training, validation, and test data.} Training, validation, and test paths are simulated independently from the same jump-diffusion process. Each path is started from $X_0\sim\mathcal{N}(0,I_d)$ and evolved for $100$ Euler steps of size $\Delta=0.01$ and $T=1$. For a given dimension $d$ and training size $N\in\{100,200,500,1000,2000,5000\}$, we draw $N$ training paths, $3000$ validation paths, and $3000$ test paths. Conditions and increments are standardized using the training mean and standard deviation. Generated paths used for evaluation are rolled out from an independent set of $3000$ initial states drawn from $\mathcal{N}(0,I_d)$. We repeat this protocol over $20$ independent replications. Our estimator and the estimator of \cite{gao2025generating} are trained and evaluated on the same simulated paths in each replication.

\paragraph{Hyperparameter tuning.}
Our estimator and the estimator of \cite{gao2025generating} are trained for at most $256$ epochs using Adam with a learning rate of $10^{-3}$. Conditions and increments are standardized using the mean and standard deviation computed from the training set. Validation performance is evaluated every $4$ epochs, and the checkpoint with the best validation score is retained. The validation score is the DSM loss computed with a frozen noise draw. Our estimator is trained on NVIDIA GeForce RTX~5090 GPUs using minibatches of size $256$.  

Since the official code for \cite{gao2025generating} is not publicly available,  we implement their estimator following the description in their paper.  As specified therein, we train one conditional score network for each time step, using an MLP with hidden widths $(32,32,16,16)$ and Softplus activations. Checkpoints are selected by the $k$NN KL divergence between real and generated paths, following the protocol in \cite{gao2025generating}.

\paragraph{Detailed values of Figure \ref{fig:our-method-vs-gao-2}.}
Table~\ref{tab:kl-our-vs-gao} reports the numerical values plotted in Figure~\ref{fig:our-method-vs-gao-2}. For each replication, we sample $20/d$ distinct times without replacement from the $100$ time steps and concatenate the corresponding $d$-dimensional states to form a $20$-dimensional joint feature. We repeat this random selection $100$ times and compute, for each selection, the $k$NN estimate of $\mathrm{KL}(\mathrm{test}\Vert\mathrm{generated})$ with $k=1$. Negative $k$NN estimates are truncated at zero, and the resulting $100$ estimates are then averaged to obtain the KL estimate for that replication. Each table entry reports the mean of these replication-level estimates over $20$ independent replications, together with a $95\%$ Student's $t$ confidence interval.


\begin{table}[ht]
\centering
\caption{Joint KL estimates (mean $\pm$ $95\%$ CI half-width) for our estimator and the estimator of \cite{gao2025generating}.}
\label{tab:kl-our-vs-gao}
\begin{tabular}{crrcc}
\hline
$d$ & $N$ & Our estimator & Estimator of~\cite{gao2025generating} \\
\hline
$1$ & $100$  & $0.165\pm 0.093$ & $3.265\pm 0.153$ \\
$1$ & $200$  & $0.140\pm 0.086$ & $3.120\pm 0.186$ \\
$1$ & $500$  & $0.087\pm 0.037$ & $3.055\pm 0.127$ \\
$1$ & $1000$ & $0.095\pm 0.032$ & $2.923\pm 0.106$ \\
$1$ & $2000$ & $0.075\pm 0.033$ & $2.685\pm 0.069$ \\
$1$ & $5000$ & $0.050\pm 0.019$ & $2.333\pm 0.065$ \\
\hline
$2$ & $100$  & $0.429\pm 0.084$ & $4.180\pm 0.148$ \\
$2$ & $200$  & $0.281\pm 0.068$ & $4.092\pm 0.118$ \\
$2$ & $500$  & $0.248\pm 0.085$ & $3.828\pm 0.119$ \\
$2$ & $1000$ & $0.202\pm 0.052$ & $3.683\pm 0.113$ \\
$2$ & $2000$ & $0.138\pm 0.038$ & $3.597\pm 0.076$ \\
$2$ & $5000$ & $0.141\pm 0.038$ & $3.411\pm 0.078$ \\
\hline
\end{tabular}
\end{table}

Table~\ref{tab:kl-our-vs-gao} shows that our estimator has lower mean KL estimates and narrower confidence intervals in all settings. 

\subsection{Additional Details for Section \ref{subsec:real-data}}

We consider the following additional datasets for the forecasting experiments.
\begin{itemize}
\item Stock:  daily closing prices from the stock series distributed with TimeGAN from \cite{NEURIPS2019_c9efe5f2}.
\item Energy: appliance and light energy use from the UCI Appliances Energy Prediction data from \cite{appliances_energy_prediction_374},
recorded every ten minutes from 11~January 2016 to 27~May 2016.
\end{itemize}

Forecasting results for these datasets, together with those for SOL and XRP in the \textbf{Crypto} dataset, are shown in Figure~\ref{fig:real-appendix}.

\begin{figure}[ht]
    \centering
    \begin{subfigure}{0.48\linewidth}
        \centering
        \includegraphics[width=0.9\linewidth]{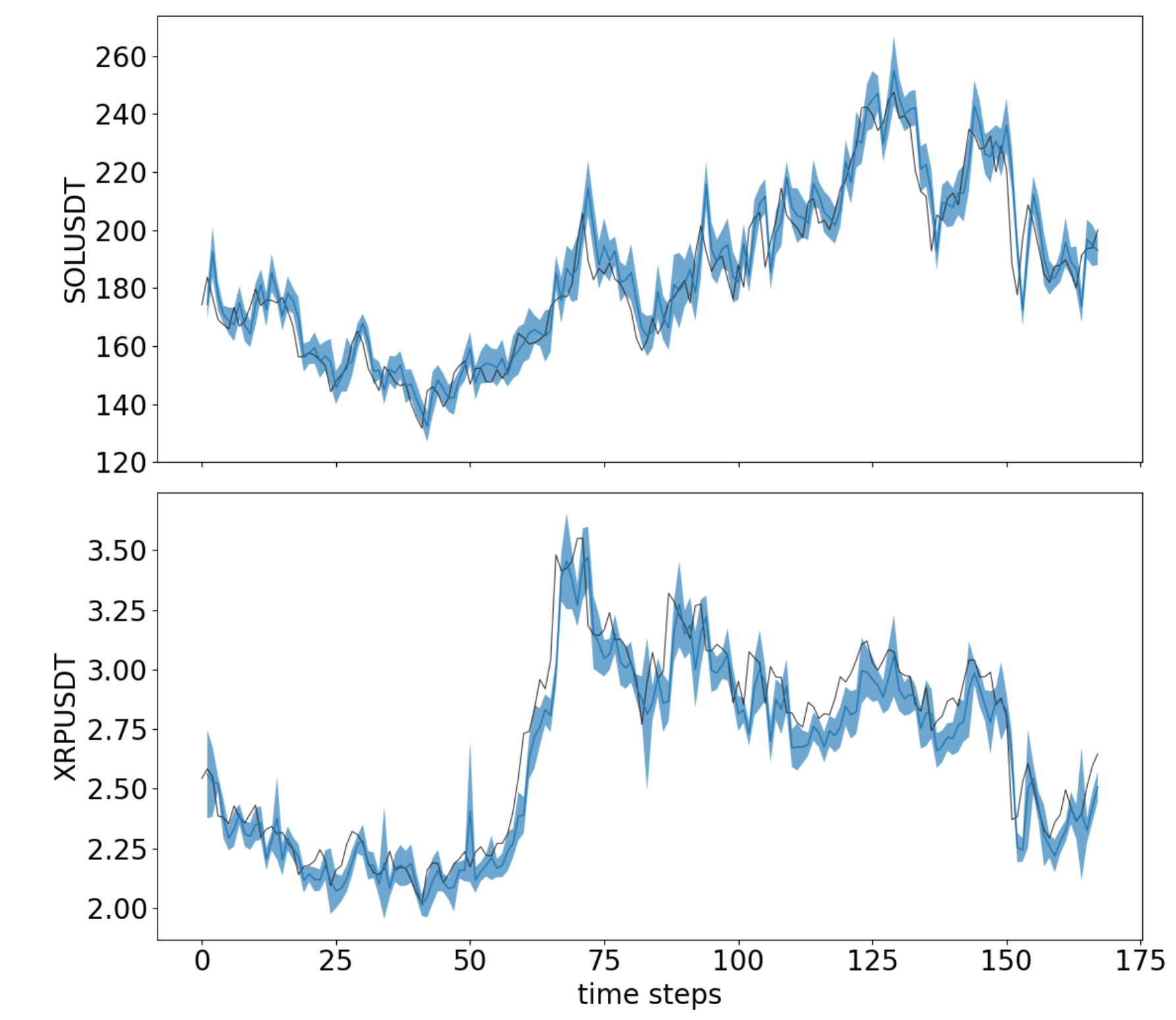}
        \caption{Crypto.}
        \label{fig:flat-crypto}
    \end{subfigure}
 \begin{subfigure}{0.48\linewidth}
        \centering
        \includegraphics[width=0.9\linewidth]{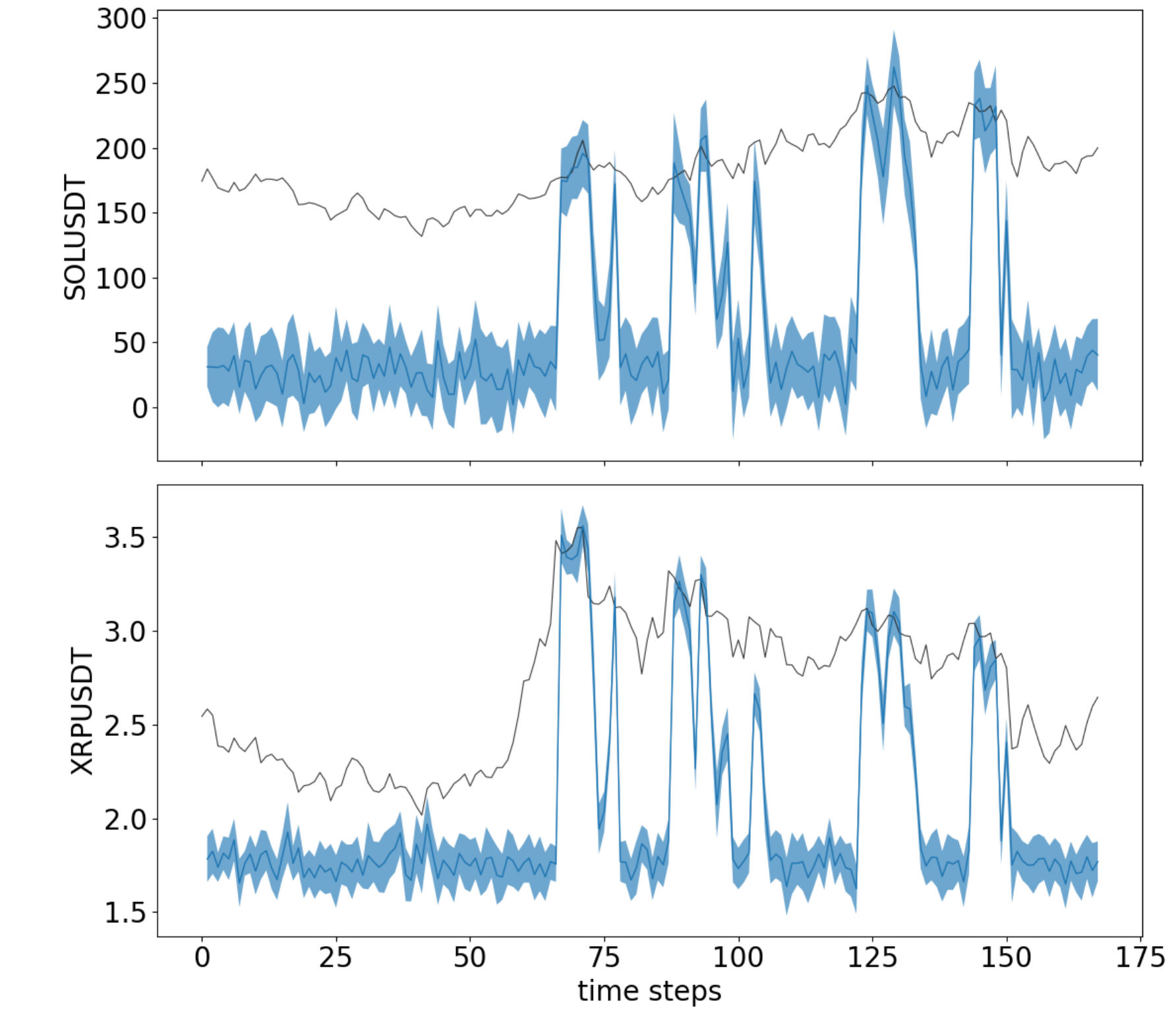}
        \caption{Crypto.}
        \label{fig:sbjts-crypto}
    \end{subfigure}

    \begin{subfigure}{0.48\linewidth}
        \centering
        \includegraphics[width=0.9\linewidth]{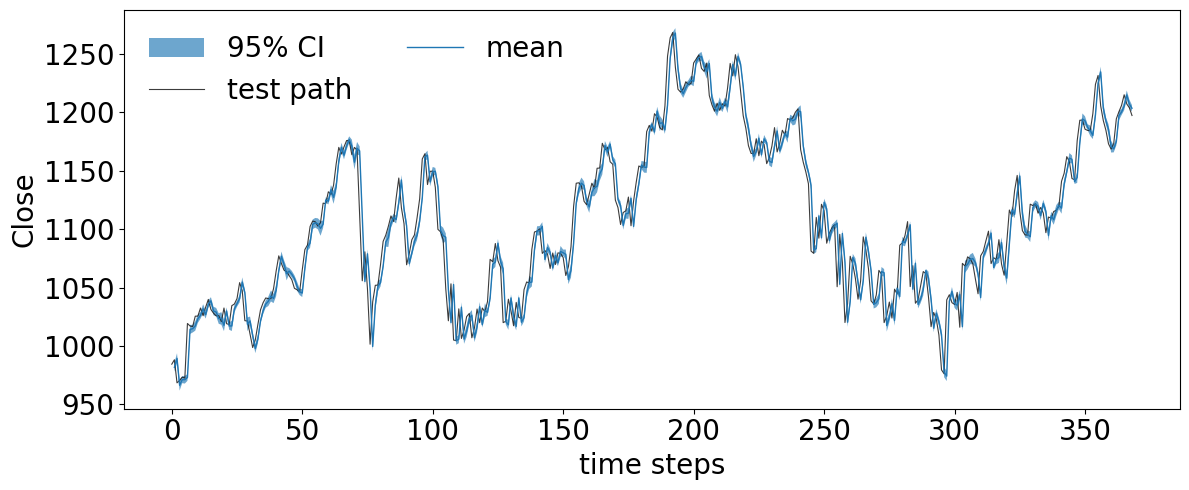}
        \caption{Stock.}
        \label{fig:flat-stock}
    \end{subfigure}
 \begin{subfigure}{0.48\linewidth}
        \centering
        \includegraphics[width=0.9\linewidth]{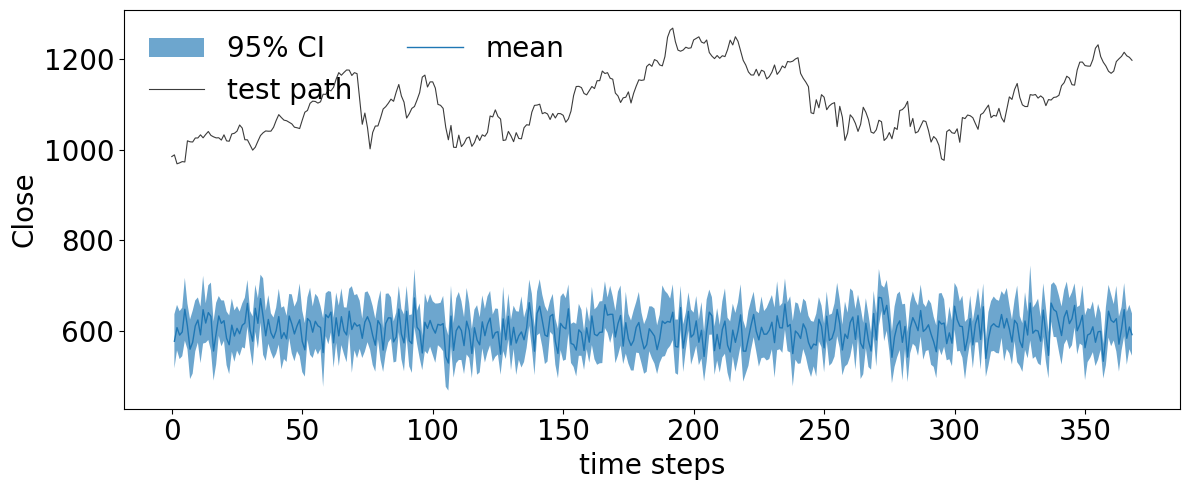}
        \caption{Stock.}
        \label{fig:sbjts-stock}
    \end{subfigure}

    \begin{subfigure}{0.48\linewidth}
        \centering
        \includegraphics[width=0.9\linewidth]{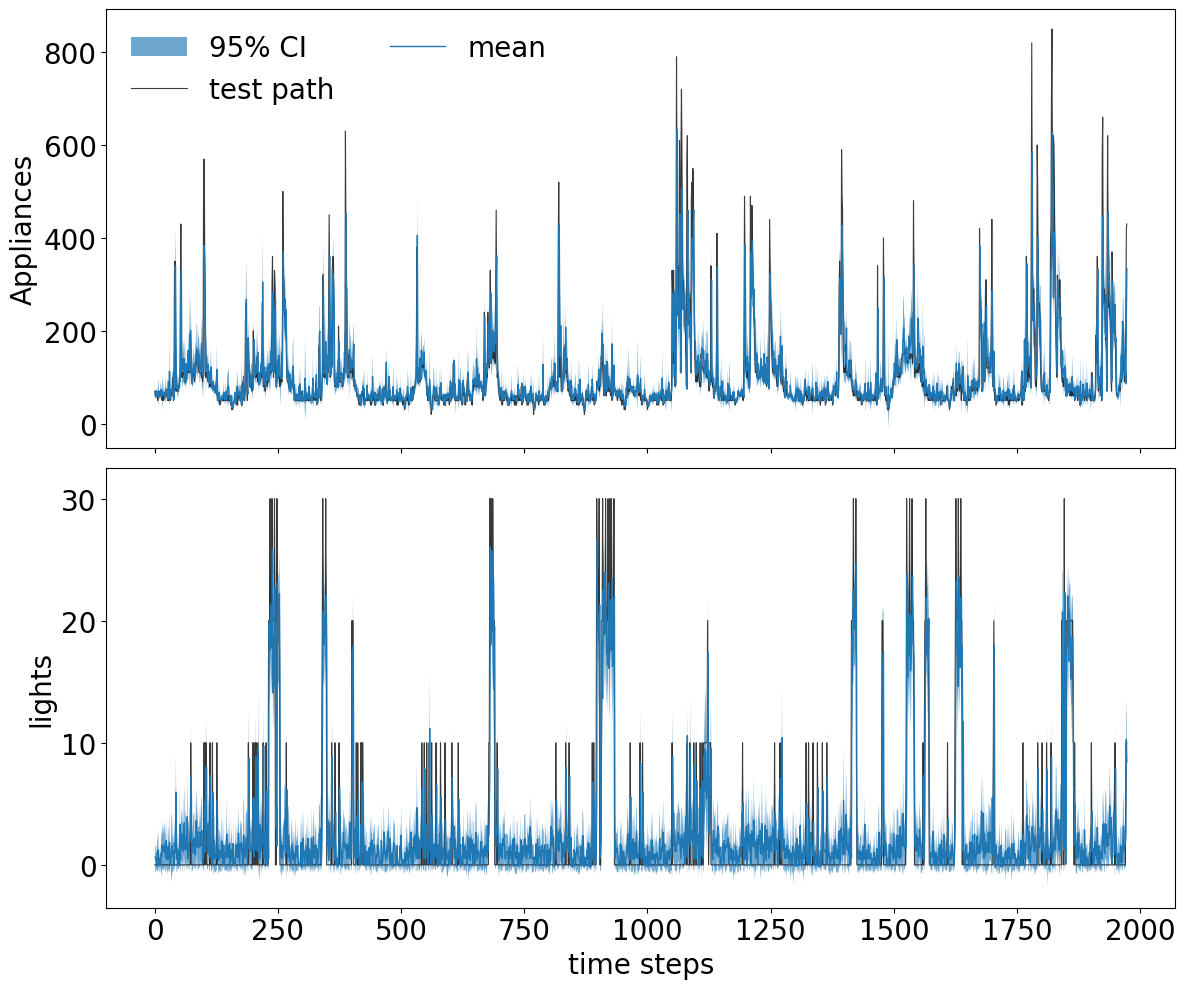}
        \caption{Energy.}
        \label{fig:flat-energy}
    \end{subfigure}
      \begin{subfigure}{0.48\linewidth}
        \centering
        \includegraphics[width=0.9\linewidth]{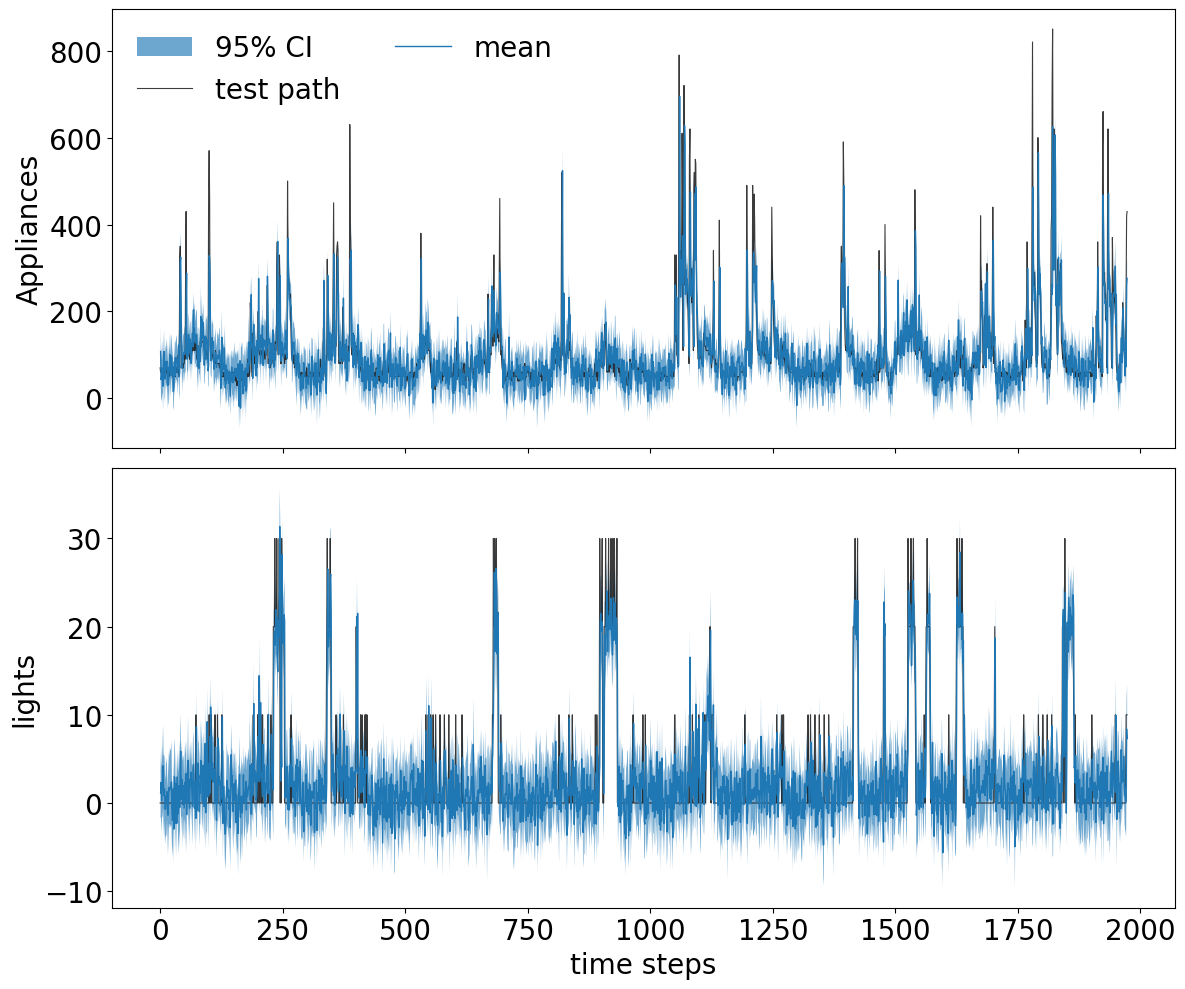}
        \caption{Energy.}
        \label{fig:sbjts-energy}
    \end{subfigure}
    \caption{One-step forecasts of our estimator (left) and time-homogeneous SBJTS estimator from \cite{demarco2026} (right) on the test segment.}
    \label{fig:real-appendix}
\end{figure}


\paragraph{Additional Details on Hyperparameter Tuning.} Our estimator uses the same hyperparameter space and training procedure as in Section~\ref{subsec:numerical-cvg}.
For the SBJTS estimator of \cite{demarco2026}, we estimate $\lambda_0$ from the empirical increment variance. The bandwidth $h$ is selected from $\{0.1,1.0,5.0\}$  by minimizing
the mean squared error on the validation segment, with the remaining SBJTS coefficients held at $\sigma=\gamma=1$, $c=0$, $T=1$, $n_{\pi}=1$, and $n_j=20$.

\paragraph{Experimental results for \textbf{Stock} Dataset.}
Figures~\ref{fig:flat-stock} and~\ref{fig:sbjts-stock} show one-step forecasts on the test segment of the one-dimensional closing price series. Our estimator tracks the observed path, with a narrow band around the test trajectory. The SBJTS estimator stays at a lower level and does not track the test trajectory.

\paragraph{Experimental results for \textbf{Energy} Dataset.}
Figures~\ref{fig:flat-energy} and~\ref{fig:sbjts-energy} show one-step forecasts on the test segment of the two-dimensional (appliances and lights) time series. Our estimator follows the spike pattern of both series. The SBJTS forecasts are noisier and the predictive band extends below zero on the lights coordinate.

\paragraph{Data sources and licensing.}
\begin{itemize}
    \item The \textbf{Crypto} prices are obtained from the Binance
Spot API.
Binance's documentation refers users to its applicable
terms of use. 
\item The \textbf{Danish} records are obtained from the
\texttt{danishuni} dataset in the \texttt{fitdistrplus} package
\citep{delignette2015fitdistrplus}.
The package is distributed under the GNU General Public
License, version 2 or later; its dataset documentation
does not specify a separate data license.
\item 
The \textbf{Stock} series is distributed with the TimeGAN
codebase \citep{NEURIPS2019_c9efe5f2}, whose documentation
identifies Yahoo Finance as the original data source.
The repository includes an Apache License 2.0 notice;
we did not identify a separate license for the underlying
stock data.
\item The \textbf{Energy} dataset
\citep{appliances_energy_prediction_374} is distributed by the
UCI Machine Learning Repository under the Creative Commons
Attribution 4.0 International license (CC BY 4.0).
\end{itemize}

\end{document}